\pdfoutput=1

\documentclass{article} 

\usepackage{times}
\usepackage{paper}
\usepackage[numbers]{natbib}
\usepackage[final]{neurips_2021}

\usepackage{amsmath,amsfonts,bm}

\def\eqref#1{equation~\ref{#1}}

\def\1{\bm{1}}

\DeclareMathAlphabet{\mathsfit}{\encodingdefault}{\sfdefault}{m}{sl}
\SetMathAlphabet{\mathsfit}{bold}{\encodingdefault}{\sfdefault}{bx}{n}

\DeclareMathOperator*{\argmax}{arg\,max}

\usepackage{booktabs} 
\usepackage{multirow}
\usepackage{adjustbox}
\usepackage{longtable}
\usepackage{hyperref}
\usepackage{url}
\usepackage{xcolor}
\usepackage{caption}

\newcommand{\acc}{\mathrm{acc}}
\newcommand{\Dtr}{D}
\newcommand{\Dtest}{D^{\text{test}}}
\newcommand{\ens}{\mathcal{T}}
\newcommand{\err}{e}
\newcommand{\cov}{ \mathrm{Cov}}
\newcommand{\ar}{ \mathrm{ar}}
\newcommand{\ensmodel}{\{h_i \}}
\newcommand\blfootnote[1]{%
  \begingroup
  \renewcommand\thefootnote{}\footnote{#1}%
  \addtocounter{footnote}{-1}%
  \endgroup
}

\makeatletter
\newenvironment{myframework}[1][htb]{%
    \renewcommand{\ALG@name}{Framework}
  \begin{algorithm}[#1]%
  }{\end{algorithm}}
\makeatother
\newcommand{\mypara}[1]{\noindent\textbf{#1}}

\newenvironment{proofsketch}{%
  \proof}{\endproof}

\title{Detecting Errors and Estimating Accuracy on Unlabeled Data with Self-training  Ensembles}

\author{%
  Jiefeng Chen \thanks{Part of the work done while interning at Google.} \\
  Department of Computer Science\\
  University of Wisconsin-Madison \\
  Madison, WI 53706 \\
  \texttt{jiefeng@cs.wisc.edu} \\
   \And
   Frederick Liu \\
   Google \\
   Seattle, WA 98103 \\
   \texttt{frederickliu@google.com} \\
   \AND
   Besim Avci \\
   Google \\
   Seattle, WA 98103 \\
   \texttt{besim@google.com} \\
   \And
    Xi Wu \\
   Google \\
   Madison, WI 53703 \\
   \texttt{wuxi@google.com} \\
   \And
   Yingyu Liang \\
   Department of Computer Science\\
   University of Wisconsin-Madison \\
   Madison, WI 53706 \\
   \texttt{yliang@cs.wisc.edu} \\
   \And
   Somesh Jha \\
   Department of Computer Science\\
   University of Wisconsin-Madison \\
   Madison, WI 53706 \\
   \texttt{jha@cs.wisc.edu} \\
}

\begin{document}

\maketitle

\begin{abstract}
When a deep learning model is deployed in the wild, it can encounter test data drawn from distributions different from the training data distribution and suffer drop in performance. For safe deployment, it is essential to estimate the accuracy of the pre-trained model on the test data. However, the labels for the test inputs are usually not immediately available in practice, and obtaining them can be expensive. This observation leads to two challenging tasks: (1) \textit{unsupervised accuracy estimation}, which aims to estimate the accuracy of a pre-trained classifier on a set of unlabeled test inputs; (2) \textit{error detection}, which aims to identify mis-classified test inputs. In this paper, we propose a principled and practically effective framework that simultaneously addresses the two tasks. The proposed framework iteratively learns an ensemble of models to identify mis-classified data points and performs self-training to improve the ensemble with the identified points. Theoretical analysis demonstrates that our framework enjoys provable guarantees for both accuracy estimation and error detection under mild conditions readily satisfied by practical deep learning models. Along with the framework, we proposed and experimented with two instantiations and achieved state-of-the-art results on 59 tasks. For example, on iWildCam, one instantiation reduces the estimation error for unsupervised accuracy estimation by at least 70\% and improves the F1 score for error detection by at least 4.7\% compared to existing methods. \blfootnote{Our code is available at: \url{https://github.com/jfc43/self-training-ensembles}.}
\end{abstract}

\section{Introduction}
\label{sec:intro}
 
Data distribution in the real world may be wildly different from the training dataset for various reasons such as covariate shift due to domain divergence, corruption of images due to weather conditions, or out-of-distribution test inputs. When facing these issues, a deployed deep learning model can have unexpected performance drop on the test data. This performance degradation can be mitigated by test data annotation, which may be costly. Hence, estimating the accuracy of a pre-trained model on the unlabeled test data provides an alternative to avoid the cost when it is not necessary. Furthermore, it is beneficial to estimate the correctness of the predictions on individual points. This leads to an even more challenging task of error detection, which aims to identify points in the unlabeled test set that are mis-classified by the pre-trained model. Such a finer-grained estimation can facilitate a further improvement of the pre-trained model (e.g., manually label those mis-classified data points and retrain the model on them).

While there have been previous attempts to address accuracy estimation or the broader problem of error detection, their successes usually rely on some conditions or assumptions that may not hold in practice.  
For example, a natural approach is to use confidence based metrics to measure the performance of the pre-trained model, e.g., as in~\citep{elsahar2019annotate}. If the model is well-calibrated on the test data, then the average confidence approximates its accuracy. However, it has been observed that many machine learning systems, in particular modern neural networks, are poorly calibrated, especially on test data with distribution shift~\citep{guo2017calibration,ovadia2019can}.
Another method is to learn a regression function that takes statistics about the model and the test data as input, and predicts the performance on the test data~\citep{elsahar2019annotate,schelter2020learning}. This requires training on labeled data from various data distributions, which is very expensive or even impractical. Furthermore, the performance predictor trained on the labeled data may not generalize to unknown data distributions. 
Recent work by~\cite{chuang2020estimating} proposes to learn a ``check'' model using domain-invariant representation and use it as a proxy for the unknown true test labels to estimate the performance via error detection. 
It relies on the success of the domain-invariant representation methods to obtain a highly accurate check model on the test data. Hence, the check model performance suffers when domain-invariant representation is not accurate in circumstances such as test data having outlier feature vectors or different class probabilities than the training data.
 
In this paper, we propose a principled and practically effective framework for the challenging tasks of accuracy estimation and error detection  (Section~\ref{sec:framework}). 
The framework makes a novel use of the self-training technique on ensembles for these tasks. It first learns an ensemble of models to identify some mis-classified points. Then it assigns pseudo-labels to these points and uses self-training with these pseudo-labeled data to identify more mis-classified points. 
We also provide provable guarantees for the framework (Section~\ref{sec:analysis}). Our analysis shows that it provably outputs an accurate estimate of the accuracy as well as the mis-classified points, under mild practical conditions: the ensemble make small errors on the test inputs correctly classified by the model $f$, mostly disagree with $f$ on the current identified mis-classified inputs, and are correct or have diverse predictions on the remaining inputs. 
These conditions can be readily satisfied by deep learning model ensembles. Furthermore, they have no explicit assumptions on the test distribution and thus the framework can be instantiated by incorporating properly designed ensemble methods for different settings (Section~\ref{sec:method}). 
Experimental results on 59 tasks over five dataset categories including image classification and sentiment classification datasets show that our method achieves state-of-the-art on both accuracy estimation and error detection (Section~\ref{sec:experiment}). 
The experiments also provide positive support for our analysis, verifying the conditions and implications. 

\section{Related Work}
\label{sec:related}

\mypara{Confidence Estimation and Error Detection. } Recently, estimating model confidence has been an important area of research because of the perceived relationship between model uncertainty and trusting its predictions \cite{piano2020}. However, modern neural networks have been observed to be poorly calibrated (e.g. they may make wrong predictions with very high confidence)~\cite{guo2017calibration}, especially on distributions with dataset shift \cite{ovadia2019can}. Many approaches have been proposed to address this issue, such as \textit{Temperature Scaling} \cite{guo2017calibration}, \textit{Monte-Carlo Dropout} \cite{gal2016dropout} and \textit{Deep Ensemble} \cite{lakshminarayanan2017simple}, but the challenge still remains. A similar challenge appears in error detection, where the goal is to identify erroneous predictions given a test set. Combining these two problems, some early work uses confidence estimates to detect incorrect predictions. For example, \cite{hendrycks2016baseline} proposed to use maximum softmax probability to detect misclassified examples. \cite{corbiere2019addressing} proposed \textit{True Class Probability} for failure prediction. \cite{jiang2018trust} proposed to use \textit{Trust Score} to estimate the confidence in model predictions. These methods require a robust estimate of confidence, and an empirically chosen threshold that is mostly problem- and model-dependent, to identify error data points. Recently, \cite{chuang2020estimating} proposed to use a check model to predict mis-classification.

\mypara{Unsupervised Accuracy Estimation. } The problem of unsupervised accuracy (or risk) estimation has received relatively scant attention from the research community. \cite{donmez2010unsupervised} offered a solution with certain assumptions on the marginal output distribution $p(y)$. \cite{platanios2014estimating} proposed to estimate the accuracies of the approximations to some target Boolean functions based on the agreement rates method. A follow-up work by \cite{platanios2017estimating} also considered the ``multiple approximations'' problem setting where the target classes may be tied together through logical constraints, and proposed an efficient method to estimate the accuracy of classifiers using only unlabeled data. \cite{JaffeNK15} proposed a spectral-based approach to estimate the accuracies of multiple classifiers, mainly in the binary case and with some assumptions on the classifiers. \cite{steinhardt2016unsupervised} proposed a method to estimate the model’s error on distributions different from the training distribution by assuming a conditional independence structure of the data. These problem settings are different from ours and the constraints may not be satisfied by the data and model we consider. Recently, \cite{elsahar2019annotate} studied three families of methods ($\mathcal{H}$-divergence, reverse classification accuracy and confidence measures) and demonstrated how they could be used to predict the performance drop. \cite{schelter2020learning} also proposed to learn a performance predictor on the statistics of model outputs with an assumption that they know the typical cases of dataset shift in advance. Similarly, \cite{DengZ21} proposed to train regression models on the feature statistics of the datasets sampled from a meta-dataset to predict model performance and \cite{DengG021} proposed to utilize linear regression to estimate classifier performance from the accuracy of rotation prediction. \cite{Devin2021} proposed to use difference of confidences to estimate classifier performance. \cite{ChenGSPFR21} proposed MANDOLINE that utilizes user-specified slicing functions to improve the importance of weighting to make the accuracy estimation more accurate. With the most similar setup to ours, \cite{chuang2020estimating} used a set of domain-invariant predictors as a proxy for the unknown target labels to estimate a given model's performance under distribution shift. In a concurrent work, \cite{jiang2021assessing} empirically showed that the test accuracy of deep networks can be estimated by measuring disagreement rate between a pair of models independently trained via Stochastic Gradient Descent (SGD) and theoretically related this phenomenon to the well-calibrated nature of ensembles of SGD-trained models.

\section{Problem Statement}
\label{sec:problem-statement} 

Consider the classification problem with sample space $\mathcal{X}$ and label space $\mathcal{Y}$. 
Let $\Dtr$ 
be a set of labeled data from the training distribution $P_{X, Y}$. Let $U = \{(x, y_x)\}$ be a set of labeled data from the test distribution $Q_{X,Y}$, where $y_x$ is the label for $x$.
Let $U_X = \{x: (x, y_x) \in U\}$  be the set of input feature vectors in $U$. Define the accuracy of a model $f$ on $U$ as: $\acc(f, U):= \frac{1}{|U|} \sum_{(x, y_x) \in U} \mathbb{I}[f(x) = y_x]$, where $|U|$ is the cardinality of the set $U$. When clear from the context, we will omit $U$ and simply write $\acc(f)$. Given a model $f:\mathcal{X}\rightarrow \mathcal{Y}$ trained on $\Dtr$, together with $\Dtr$ and $U_X$, the goal of \textit{unsupervised accuracy estimation} is to get an estimate $\hat{\acc}$ so that the absolute estimation error $|\hat{\acc} - \acc(f)|$ is small. We assume access to the training data $\Dtr$ to help in estimating $\acc(f)$.


Though $\acc(f)$ is usually sufficient to provide an estimate of whether the model could perform well on $U$, it would be beneficial if the algorithm could also provide an estimate of the correctness of individual points so that we can know where to improve the model $f$. Thus, we also consider a more challenging task of identifying the mis-classified points in $U$: $W_X := \{x: (x, y_x) \in U, f(x)\neq y_x \}$. Given $f$, $\Dtr$ and $U_X$, the goal of \textit{error detection} is to identify a subset $R_X \subseteq U_X$, such that $|W_X \triangle R_X|$ is small, where $W_X \triangle R_X:=(W_X \setminus R_X) \cup (R_X \setminus W_X)$.

\section{Algorithmic Framework via Self-Training Ensembles} \label{sec:framework}

\mypara{Intuition.}
We consider the approach of learning a ``check model'' $h$ and using the disagreement between $h$ and the pre-trained model $f$ for our tasks. The fundamental idea is to identify a point $x$ as mis-classified if $h$ disagrees with $f$ on $x$.  
To derive our method, we note a simple fact: the disagreement approach succeeds if (1) $h$ agrees with $f$ on points where $f$ is correct and (2) $h$ disagrees with $f$ on points where $f$ is incorrect. Our first key observation is that usually (1) can be satisfied approximately in practice. Intuitively, this is because $h$ and $f$ use the same training data, and $h$ can be trained to be correct on the subset of the instance space where $f$ is correct. However, (2) may not be easily satisfied (e.g., $h$ can make similar mistakes as $f$), which leads to an overestimation of the accuracy. 
We thus focus on improving the disagreement on mis-classified points. 

To disagree with $f$ on a mis-classified test input $x$, we have two ways: (1) make the check model $h$ correct on $x$; (2) diverse ensembles.
The first way may be achievable when the training data contains information for prediction on $x$. A prototypical example is when the test inputs are corruption of clean data from the training distribution (e.g., the training data are images in sunny days while the test inputs are ones in rainy days), and techniques like unsupervised domain adaptation can be used to improve the prediction on such test inputs. However, correct predictions on $x$ may not be feasible in many interesting scenarios due to insufficient information (e.g., 
the test image in the open world can contain an object that is never seen in the training data
). Fortunately, it has been shown that for such inputs, one can obtain an ensemble of models with diverse predictions (e.g.,\cite{lakshminarayanan2017simple}). This then gives the second way to achieve disagreement: using diverse ensembles. 
Therefore, our method will learn an \emph{ensemble} of models (instead of one check model) and identify a point $x$ as mis-classified if the ensemble disagree with $f$ on $x$ (i.e., a large fraction of the models in the ensemble disagree with $f$ on their predictions on $x$). A mis-classified test input, $x$, will be successfully identified, if a majority of the ensemble models predict correctly, or they have large diversity on $x$. 

However, the ensemble may only be able to identify a subset of the mis-classified points. Therefore, we propose to iteratively identify more and more mis-classified points by \emph{self-training}. For each mis-classified data point $x$ identified by the ensemble, we assign it a pseudo-label that is different from $f(x)$ (e.g. use the majority vote of the ensemble or a random label as the pseudo-label). Then we can train (with regularization) a new ensemble to encourage their disagreement with $f$ on the pseudo-labeled data $R$ (e.g., use a supervised loss on $R$ with a small weight as the regularization). Note that the pseudo-labels may not all be correct, and we do not need the new ensemble to exactly fit the pseudo-labels. We only need the new ensemble to mostly disagree with $f$ on $R$ so that they still identify $R$ as mis-classified points. 
Furthermore, self-training can help the ensemble identify more mis-classified points. When some pseudo-labels are correct, we observe that these additional data can help the new ensemble become more accurate and thus help with identifying more mis-classified points.
We also observe that the new ensemble will be less diverse on the pseudo-labeled data and thus have diversity on the remaining test inputs. We report empirical results in Appendix~\ref{app:validate-theory} to support these observations.

\begin{myframework}[t] 
	\caption{ Error Detection and Unsupervised Accuracy Estimation via Self-Training Ensembles}
	\label{alg:main-framework}
	\begin{algorithmic}[1]
		\REQUIRE A training dataset $\Dtr$, an unlabeled test dataset $U_X$, a pre-trained model $f$, an ensemble learning method $\ens$, and a hyper-parameter $\tau$ 
		\STATE Initialize $R = \emptyset$
		\FOR{$t=1, 2, \cdots, T$}
		\STATE Use $\ens$ to generate an ensemble $\ens(\Dtr, U_X, R)$.
		\STATE Identify those points on which the ensemble and $f$ disagree as mis-classified points:
		\begin{align*}
            R_X := \{ x \in U_X:
            \Pr_{h \sim \ens}\{ h(x) = f(x)\} < \tau \}.
        \end{align*}
		\STATE For each $x \in R_X$, assign a pseudo-label $\tilde{y}_x \neq f(x)$ (e.g., using the majority vote of the ensemble or a random label). And set $R = \{(x, \tilde{y}_x): x \in R_X\}$. 
		\ENDFOR 
        \ENSURE The estimated mis-classified points $R_X $, and the estimated accuracy $\frac{|U_X \setminus R_X|}{|U_X|}$.
	\end{algorithmic}
\end{myframework}

\mypara{Our framework.} 
We assume access to an ensemble learning method $\ens$ (specific instantiations will be given later) that takes as input a training set $\Dtr$, an unlabeled test set $U_X$, and a pseudo-labeled set $R$ (and potentially some other parameters) and outputs an ensemble, which is a distribution over functions $h: \mathcal{X} \rightarrow \mathcal{Y}$.%
\footnote{This includes one single $h$ as a special case. It also includes the case of $h$ with probabilistic outputs, i.e., $h(x) = [h^1(x), \dots, h^K(x)]$ for $\mathcal{Y} = \{1, 2, \ldots, K\}$, where $h^i(x)$ is the predicted probability of $x$ from class $i$. One can think of such an $h$ as an ensemble $\ens$ where $\Pr_{g \sim \ens}[g(x) = i] = h^i(x)$.} 
We denote the generated ensemble as $\ens(\Dtr, U_X, R)$ and write $h \sim \ens(\Dtr, U_X, R)$ for sampling $h$ from the ensemble, or simply $h \sim \ens$ when clear from context.

We are now ready to describe our method (Framework~\ref{alg:main-framework}).
Our framework begins with an empty set of pseudo-labeled data $R$, and executes in $T$ iterations. In each iteration, it generates an ensemble $\ens(\Dtr, U_X, R)$ and then uses the ensemble to construct a new pseudo-labeled dataset $R$ for the next iteration: let $R_X$ be those points $x \in U_X$ where the agreement rate between the ensemble and $f$ is below a threshold, assign a pseudo-label $\tilde{y}_x$ different from $f(x)$, and let $R$ be the set of these pseudo-labeled data.  
Finally, it outputs $R_X$ as the mis-classified points and outputs the fraction of points outside $R_X$ as the accuracy.

Some existing work also use models' disagreement to estimate the error on the unlabeled data in different ways. For example,~\cite{platanios2014estimating} consider a set of pre-trained models and aim to estimate the error for each model;~\cite{chuang2020estimating} learn a single check model and compare it with $f$. On the other hand, our novelty is using \emph{ensembles} to identify a set of mis-classified points and further using \emph{self-training} for learning the  ensembles. Note that our self-training ensembles  is different from standard self-training and ensemble: standard self-training+ensemble aims to get accurate predictions while ours is to increase disagreement on mis-classified points (refer to Appendix~\ref{sec:comparison-with-standard-self-training-ensemble} for a detailed discussion).

\section{Theoretical Analysis}
\label{sec:analysis}
 
As described in the intuition above, our framework succeeds if in each iteration, the ensemble satisfy the following conditions: \textbf{(A)} correct on $U_X \setminus W_X$, \textbf{(B)} mostly disagree with $f$ on $R_X$, and \textbf{(C)} either correct or diverse on $W_X \setminus R_X$.  
This section first introduces notions formalizing these conditions and then provides provable guarantees under these conditions.  
Here we focus on intuitions and provide the proofs and more discussion in Appendix~\ref{app:additional_analysis}.
 
\mypara{Notations.}
Recall that the ensemble method $\ens$ outputs a distribution over models, and we write $h \sim \ens$ for sampling a model $h$ from this distribution. Let $\mathbb{E}_h$ denote the expectation over this output distribution, and $\mathbb{E}_x$ denote the average over $x \in U_X$ or $(x, y_x) \in U$. 

\mypara{Formalizing Condition (A).}
A key observation we make in practice is that usually on points where $f$ is correct, the ensemble models are also correct. Intuitively, this is because the only labeled information for learning the ensemble is $\Dtr$, while $\Dtr$ is also used for learning $f$. 
For illustration, suppose $f$ and the ensemble have zero training errors on $\Dtr$. When they are from hypothesis classes of bounded VC-dimensions and the training set is large enough compared to the VC-dimensions, standard error bounds show that the ensemble will have small errors on the subset of the feature space $\{x \in \mathcal{X}: f(x) = y_x\}$, i.e., it will have small errors on points where $f$ is correct. 
To formalize this observation, we introduce the following notion.

\begin{definition}[Error on Correct Points] \label{def:average_error}
Let $\nu$ denote the average probability of $h \sim \ens$ making error on test inputs where the model $f$ is correct: $\nu := \Pr_{(x,y_x) \sim U, h \sim \ens} \ [h(x) \neq y_x \mid f(x) = y_x]$.
\end{definition}

\mypara{Formalizing Condition (B).}
We assume the new ensemble trained with regularization on $R$ will mostly disagree with $f$ on $R_X$. We define the following notion:   
\begin{definition}[Agreement on Pseudo-Labeled Data] \label{def:error_pseudo}
Let $\gamma$ denote the average probability of agreement between $h \sim \ens$ and $f$ on $R_X$: $\gamma := \Pr_{x \sim  R_X, h \sim \ens} \quad \{ h(x) = f(x)  \}$.

\end{definition}

\mypara{Formalizing Condition (C).}
Let $G_X$ denote the good points in $W_X \setminus R_X$ on which the ensemble will have correct predictions with high confidence, $B_X$ the remaining bad points. Formally, 
we define: $G_X := \{ x \in W_X \setminus R_X: \Pr_{h \sim \ens}\{ h(x) = y_x \} \ge 1- \nu  \}$ and $B_X := W_X \setminus (R_X \cup G_X)$.
  
(Here we choose the confidence level $1-\nu$ for convenience, where $\nu$ is the error on correct points. Other sufficiently high confidence levels can also be used with slight change to the analysis.)
We would like the ensemble to have large diversity on $B_X$.
\begin{definition}[Diversity of Ensemble] \label{def:diversity}
Let $\sigma^2$ denote the average probability of disagreement between two ensemble models on $B_X$: $\sigma^2  := \mathbb{E}_x [\sigma_x^2 | x \in B_X]$, where $\sigma_x^2 := \Pr_{h_1, h_2 \sim \ens} [h_1(x) \neq h_2(x)]$.

\end{definition}

\mypara{Provable Guarantees.}
Based on the above notions we obtain our guarantees: 
\begin{theorem} \label{thm:main}
Assume in each iteration of the framework, $\tau = \sqrt{1 -\eta}$ for some $\eta \in (0, 3 B_\eta/4)$ where $B_\eta := \min\{\sigma^2, 1- \nu^2 \}$.
Let $\sigma_L^2 > 0$ be a lower bound on the diversity $\sigma^2$, $\tilde{\gamma}>0$ be an upper bound on $\gamma$, and $\tilde{\nu}$ be an upper bound on $\nu$ over all iterations. Then for any $\delta \in (0, \sigma_L^2/4)$, after at most $\lceil 1/\delta \rceil$ iterations, we can get $\frac{|U_X \setminus R_X|}{|U_X|}$ approximates the accuracy $\acc(f)$ and $R_X$ approximates the mis-classified points $W_X$ as follows: 
\begin{align}
    \left| \acc(f) -  \frac{|U_X \setminus R_X|}{|U_X|}\right| & \le  \max\{\frac{\tilde{\nu}}{1-\tau} (1-e_f), \epsilon e_f\}, \textrm{~where~} \epsilon := \frac{\frac{\tilde{\gamma}}{\tau}\left( 1 + \frac{\tilde{\nu}}{1-\tau} \frac{1-\err_f}{\err_f}\right)}{\frac{\sigma_L^2}{4} - \delta + \frac{\tilde{\gamma}}{\tau}},
    \\
    |W_X \triangle R_X|  & \le \frac{\tilde{\nu}}{1-\tau} |U_X \setminus W_X| + \epsilon |W_X|.
\end{align}
\end{theorem}
\begin{proofsketch}
Let's first consider one iteration, assuming small $\nu, \gamma$ and large $\sigma^2$. Intuitively, on the correct points $U_X \setminus W_X$, the ensemble have a small error $\nu$ and thus disagree with $f$ on only a few such points.
On the old $R_X$, the ensemble mostly disagree with $f(x)$; similarly on $G_X$.
On the remaining points $B_X$, we show that if the ensemble have large diversity, then their predictions must have large variances on a significant subset of points and thus have large disagreement with $f$, facilitating the detection. 
Overall, our framework can construct a new pseudo-labeled set $R'_X$ that contains mostly mis-classified points and is also larger than the old $R_X$ (Lemma~\ref{lem:constructR} in Appendix~\ref{app:additional_analysis}). 
Therefore, each iteration can make some progress by identifying more mis-classified points than before, and enough iterations achieve the guarantees.  
\end{proofsketch}

The theorem provides guarantees for general values of $\nu, \gamma$ and $\sigma^2$. To get some intuition, note that typically $\err_f < \tau$ and suppose we set $\tau = 3/4$ and $\delta=\tilde{\gamma}/\tau$, then the accuracy is estimated up to error $\max\{\tilde{\nu}, \frac{16\tilde{\gamma}}{3\sigma_L^2} (\err_f + \tilde{\nu})\}$. When $\tilde{\nu}, \tilde{\gamma}$ are small and $\sigma_L^2$ is large, the error is small. 
Therefore, under mild conditions, our framework can give a provable estimation of the accuracy and the mis-classified points up to small errors.

The theorem formalizes that the framework can succeed under mild conditions \textbf{(A)(B)(C)} on the ensembles, without explicit conditions on the data distributions and the pre-trained model (note that the conditions on them are implicitly captured in the mild conditions on the ensemble). The framework is thus flexible, and different ensemble methods satisfying these conditions can be incorporated to get concrete instantiations applicable to various settings. Indeed, the crux of our framework is then to design ensemble methods meeting the conditions. This turns out to be not difficult, in particular for deep learning pre-trained models and deep learning ensemble models. Even an ensemble of deep networks simply trained from random initialization works well. Two concrete instantiations are presented in Section~\ref{sec:method} below and evaluated in Section~\ref{sec:experiment}.

\section{Instantiations of the Framework}
\label{sec:method}

\setcounter{algorithm}{0}
\begin{algorithm}[t] 
	\caption{Error Detection and Unsupervised Accuracy Estimation via Self-Training Ensembles}
	\label{alg:main}
	\begin{algorithmic}[1]
		\REQUIRE $\Dtr$, $U_X$, $f$, ensemble method $\ens$, parameters $T$, $N$ 
		\STATE Initialize $R = \emptyset$
		\FOR{$t=1, 2, \cdots, T$}
		\STATE Set $ \{h_i\}_{i=1}^N = \ens(\Dtr, U_X, R, N, \textrm{other parameters}) $ 
		\STATE Let $\tilde{y}_x$ be the majority vote of $ \{h_i\}_{i=1}^N $:
		$ 
             \tilde{y}_x := \argmax_{j \in \mathcal{Y}} \frac{1}{N} \sum_{i=1}^N \mathbb{I}[h_i(x) = j]. 
        $ 
		\STATE Set $R = \{(x, \tilde{y}_x): x \in U_X, \tilde{y}_x \neq f(x) \}$.  
		\ENDFOR 
        \ENSURE Estimated mis-classified points $R_X = \{x: (x,y) \in R\}$, estimated accuracy $\frac{|U_X \setminus R_X|}{|U_X|}$.
	\end{algorithmic}
\end{algorithm}

Based on our framework, we propose Algorithm~\ref{alg:main} for accuracy estimation and error detection. It executes in $T$ iterations. In each iteration, it trains an ensemble of models $\{h_i\}_{i=1}^N$ and then uses the ensemble to construct the pseudo-labeled data. Finally, it outputs $R_X$ as the set of mis-classified points and outputs the fraction of points outside $R_X$ as the accuracy. The algorithm uses an intuitive heuristic to implement the threshold on the agreement rate: it sets $R_X$ as the points $x \in U_X$ with $f(x)$ different from the majority vote of the ensemble models. 
Empirically, we observe that this leads to similar $R_X$ and thus similar results as explicit thresholding, but is much more convenient (the empirical results to support this observation are reported in Appendix~\ref{sec:thresholding-implementation-results}). 

 

While different ensemble methods $\ens$ can be used in Algorithm~\ref{alg:main}, here we propose two concrete instantiations $\ens_\textrm{RI}$ and $\ens_\textrm{RM}$ based on the success conditions of our framework.

\mypara{Ensemble Method $\ens_\textrm{RI}$.}
Algorithm~\ref{alg:ensemble_ri} describes a natural method that trains the models from different random initialization. It first trains $h'_i$ on $\Dtr$ (e.g., using the same training algorithm as for $f$), to ensure that ensemble has small error on points where $f$ is correct. It then fine-tunes $h'_i$ on $\Dtr$ and $R$ for one epoch to get $h_i$ that mostly disagrees with $f$ on $R_X$. Finally, deep models trained from different random initialization have been shown to be diverse on outlier data points~\cite{lakshminarayanan2017simple,fort2019deep}. In summary, the ensemble constructed can satisfy our three conditions. 


\mypara{Ensemble Method $\ens_\textrm{RM}$.}
Algorithm~\ref{alg:ensemble_rm} describes another method designed with the representation matching technique for domain adaptation, which can potentially improve the accuracy of the ensemble on some test inputs related to the training data and thus satisfy our success condition \textbf{(C)} better. It requires the model architecture to be $c(\phi(x))$, i.e., a composition of a prediction function $c$ and a representation function $\phi$. Beginning from a pre-trained model $h_0$, 
it fine-tunes $h_0$ for $N$ epochs by minimizing the loss on $\Dtr$ and $R$ \emph{plus a representation matching loss $\alpha \cdot d(p^\phi_{\Dtr}, p^\phi_{U_X})$}, and outputs the $N$ checkpoint models at the end of each training epoch as the ensemble. A key component of $\ens_\textrm{RM}$ is representation matching, which aims to learn a function $\phi$ that minimizes $d(p^\phi_{\Dtr}, p^\phi_{U_X})$. 
It has been shown that representation matching can improve the accuracy on the test data from the target domain. Also, we have observed that the checkpoint models can have diversity on the mis-classified data points empirically. Thus, the ensemble constructed can satisfy our conditions. For our experiments, we use the representation matching loss from the classic DANN~\citep{ganin2016domain}.  

\begin{algorithm}[t] 
	\caption{ $\ens_\textrm{RI}$: Ensemble via Random Initialization}
	\label{alg:ensemble_ri}
	\begin{algorithmic}[1]
		\REQUIRE $\Dtr$, $U_X$, $R$, $N$, parameter $\gamma$
		\STATE Pre-train $\{h'_i\}_{i=1}^N$ on $\Dtr$ from different random initialization
		\FOR{$i = 1, 2, \ldots, N$} 
		\STATE Learn $h_i$ by fine-tuning $h'_i$ for one epoch by:
		\begin{align}
		    \label{obj:ri}
		    	\hspace{-.4in}\minimize_{h}  \mathbb{E}_{(x,y)\in \Dtr } [\ell(h(x),y)] 
		    	 + \gamma \cdot \mathbb{E}_{(x,y)\in R } [\ell(h(x),y)]
		\end{align}
		\ENDFOR
		\ENSURE The ensemble of models $\{h_i\}_{i=1}^N$ 
	\end{algorithmic}
\end{algorithm}

\begin{algorithm}[t] 
	\caption{ $\ens_\textrm{RM}$: Ensemble via Representation Matching}
	\label{alg:ensemble_rm}
	\begin{algorithmic}[1]
		\REQUIRE $\Dtr$, $U_X$, $R$, $N$, initial pre-trained model $h_0$, parameters $\alpha, \gamma$   
		
		\textit{// $h_0(x)=c(\phi(x))$ is a composition of a prediction function $c$ and a representation function $\phi$ 
		}
		\STATE Fine-tune $h_0$ for $N$ epochs using the objective: 
		\begin{align}
		    \label{obj:rm}
		    	 \minimize_{h} \mathbb{E}_{(x,y)\in \Dtr } [\ell(h(x),y)] 
		    	+  \gamma \cdot \mathbb{E}_{(x,y)\in R } [\ell(h(x),y)] + \alpha \cdot d(p^\phi_{\Dtr}, p^\phi_{U_X})
		\end{align}
		where $d(p^\phi_{\Dtr}, p^\phi_{U_X})$ is the distance between the distribution of $\phi(x)$ on $\Dtr$ and that on $U_X$. 
		\STATE Use the $N$ checkpoint models at the end of each training epoch as the model ensemble $\{ h_i \}_{i=1}^N$. 
		\ENSURE The ensemble of models $\{h_i\}_{i=1}^N$ 
	\end{algorithmic}
\end{algorithm}

\section{Experiments}
\label{sec:experiment}


We perform experiments for unsupervised accuracy estimation and error detection tasks on 59 pairs of training-test datasets from five dataset categories, including image classification and sentiment classification datasets. In summary, our findings are: \textbf{(1)} Our method achieves state-of-the-art results on both accuracy estimation and error detection tasks than the existing methods. \textbf{(2)} Both ensemble and self-training techniques have positive effects on the tasks and it is easy to pick suitable hyper-parameters for our algorithms. \textbf{(3)} Empirical results show that the conditions made in our analysis hold approximately.

\subsection{Setup}


We briefly describe the experiment setup here. The detailed setup can be found in Appendix~\ref{sec:detail-experiment-setup}. 

\mypara{Dataset. } The task needs a pair of training-test datasets $\Dtr$ and $U_X$. We use five dataset categories, each containing multiple training-test dataset pairs. Specifically, we use the following dataset categories: Digits (including MNIST~\citep{lecun1998mnist}, MNIST-M~\citep{ganin2016domain}, SVHN~\citep{netzer2011reading}, USPS~\citep{hull1994database}), Office-31~\citep{saenko2010adapting}, CIFAR10-C~\citep{krizhevsky2009learning}, iWildCam~\citep{beery2020iwildcam} and Amazon Review~\citep{blitzer7domain}. Digits has 12 dataset pairs, Office-31 has 6, CIFAR10-C has 19, iWildCam has 10 and Amazon Review has 12. 

\mypara{Evaluation Metrics. } We use absolute estimation error for accuracy estimation and use F1 score for error detection.

\mypara{Our Models. } On each dataset category, we design a neural network architecture for the DANN training algorithm~\citep{ganin2016domain}, which is named DANN-arch. It contains an encoder, a predictor branch and a discriminator branch. For $\ens_\textrm{RM}$, we use the DANN-arch for the ensemble $\ensmodel$ and pre-train the initial model $h_0$ on $\Dtr$ and $U_X$ using DANN algorithm. For $\ensmodel$ in $\ens_\textrm{RI}$, we also use the DANN-arch. The model $f$ is pre-trained on $D$ and we mainly consider two kinds of architectures for it: one is the DANN-arch (i.e., the model $f$ shares the same architecture as the check model $h$); the other is a typical deep neural network (DNN). 


\mypara{Hyper-parameters. } The hyper-parameters $T$, $N$ and $\gamma$ can be easily selected since a broad range of values can lead to good results. Based on our observation, larger $T$ and $N$ can lead to better results. In our experiments, we set $T=5$ and $N=5$ by considering the computational cost (on Amazon Review, we set $N=20$). We set $\gamma=0.1$ and set $\alpha$ following the domain adaptation methods. 


\mypara{Baselines. } For accuracy estimation, we consider Proxy Risk~\citep{chuang2020estimating}, Average Confidence (Avg Conf)~\citep{elsahar2019annotate}, and Ensemble Average Confidence (Ens Avg Conf). For error detection, we consider Proxy Risk, Maximum Softmax Probability (MSP)~\citep{hendrycks2016baseline}, and Trust Score~\citep{jiang2018trust}. Although Proxy Risk and our method share some similar ideas such as the use of check models, disagreement and representation matching, they have some major differences in the key ideas, training objectives, and the implementation of training objectives (refer to Appendix~\ref{sec:comparison-with-proxy-risk} for a detailed discussion).

\subsection{Results}

\begin{table}
\centering
\begin{adjustbox}{width=0.9\columnwidth,center}
		\begin{tabular}{l|l|c|c|l|c|c}
			\toprule
			Task & \multicolumn{3}{|c}{Accuracy Estimation} & \multicolumn{3}{|c}{Error Detection} \\ \hline
			Metric & \multicolumn{3}{c|}{Absolute Estimation Error $\downarrow$} & \multicolumn{3}{c}{F1 score $\uparrow$} \\ \hline
			\multirow{2}{0.12\linewidth}{Dataset} & \multirow{2}{0.12\linewidth}{Method} & \multicolumn{2}{c|}{Model $f$} & \multirow{2}{0.12\linewidth}{Method} & \multicolumn{2}{c}{Model $f$} \\ \cline{3-4} \cline{6-7} 
			&  & Typical DNN  & DANN-arch &  & Typical DNN  & DANN-arch \\ \hline \hline 
	         \multirow{4}{0.12\linewidth}{Digits} & Avg Conf & 0.404$\pm$0.180 & 0.350$\pm$0.230  & MSP & 0.467$\pm$0.195 & 0.485$\pm$0.209  \\ 
	       & Ens Avg Conf  & 0.337$\pm$0.229 & 0.246$\pm$0.230 & Trust Score & 0.496$\pm$0.195 & 0.484$\pm$0.187 \\ 
	       & Proxy Risk  & 0.085$\pm$0.142 & 0.095$\pm$0.181 &  Proxy Risk & 0.844$\pm$0.118 & 0.796$\pm$0.155 \\
	       & Ours (RI) & 0.164$\pm$0.218 & 0.087$\pm$0.077 & Ours (RI) & 0.698$\pm$0.235 & 0.701$\pm$0.126 \\
	       & Ours (RM) & {\bf 0.023}$\pm$0.020 & {\bf 0.024}$\pm$0.022 & Ours (RM) & {\bf 0.881}$\pm$0.084 & {\bf 0.841}$\pm$0.112 \\ \hline 
	      \multirow{4}{0.12\linewidth}{Office-31}& Avg Conf  & 0.054$\pm$0.044 & 0.259$\pm$0.134  & MSP & 0.281$\pm$0.266 & 0.584$\pm$0.128  \\ 
	       & Ens Avg Conf & 0.080$\pm$0.041 & 0.281$\pm$0.136 & Trust Score & 0.401$\pm$0.240 & 0.559$\pm$0.143 \\ 
	       & Proxy Risk & 0.033$\pm$0.012 & 0.042$\pm$0.034 &  Proxy Risk & 0.605$\pm$0.177 & 0.629$\pm$0.140 \\
	       & Ours (RI) & 0.051$\pm$0.038 & 0.044$\pm$0.031 & Ours (RI) & 0.715$\pm$0.124 & 0.770$\pm$0.027 \\
	       & Ours (RM)  & {\bf 0.029}$\pm$0.021 & {\bf 0.018}$\pm$0.023 & Ours (RM) & {\bf 0.767}$\pm$0.052 & {\bf 0.790}$\pm$0.087 \\ \hline 
	       \multirow{4}{0.12\linewidth}{CIFAR10-C} & Avg Conf & 0.353$\pm$0.175 & 0.369$\pm$0.176 & MSP & 0.505$\pm$0.043 & 0.550$\pm$0.043 \\ 
	       & Ens Avg Conf & 0.237$\pm$0.144 & 0.237$\pm$0.133 & Trust Score & 0.494$\pm$0.045 & 0.568$\pm$0.060 \\ 
	       & Proxy Risk & 0.053$\pm$0.070 & 0.052$\pm$0.070 &  Proxy Risk & 0.850$\pm$0.107 & 0.843$\pm$0.101 \\
	       & Ours (RI) & 0.149$\pm$0.089 & 0.197$\pm$0.115 & Ours (RI) & 0.654$\pm$0.064 & 0.568$\pm$0.063 \\
	       & Ours (RM) & {\bf 0.022}$\pm$0.009 & {\bf 0.029}$\pm$0.012 & Ours (RM) & {\bf 0.872}$\pm$0.083 & {\bf 0.860}$\pm$0.091  \\ \hline 
	       \multirow{4}{0.12\linewidth}{iWildCam} & Avg Conf & 0.388$\pm$0.045 & 0.395$\pm$0.043 & MSP & 0.692$\pm$0.006 & 0.741$\pm$0.009  \\ 
	       & Ens Avg Conf & 0.177$\pm$0.025 & 0.158$\pm$0.020 & Trust Score & 0.717$\pm$0.009 & 0.737$\pm$0.010 \\ 
	       & Proxy Risk & 0.119$\pm$0.043 & 0.094$\pm$0.036 &  Proxy Risk & 0.755$\pm$0.038 & 0.773$\pm$0.039 \\
	       & Ours (RI) & {\bf 0.015}$\pm$0.008 & {\bf 0.007}$\pm$0.004 & Ours (RI) & 0.792$\pm$0.013 & 0.806$\pm$0.012 \\
	       & Ours (RM) & 0.035$\pm$0.022 & 0.026$\pm$0.024 & Ours (RM) & {\bf 0.796}$\pm$0.014 & {\bf 0.809}$\pm$0.010 \\  \hline
	       \multirow{4}{0.12\linewidth}{Amazon Review} & Avg Conf & 0.290$\pm$0.043 & 0.310$\pm$0.045 & MSP & 0.420$\pm$0.022 & 0.218$\pm$0.031  \\ 
	       & Ens Avg Conf & 0.229$\pm$0.038 & 0.217$\pm$0.036 & Trust Score & 0.414$\pm$0.024 & 0.237$\pm$0.044 \\ 
	       & Proxy Risk & 0.021$\pm$0.014 & 0.037$\pm$0.076 &  Proxy Risk & 0.417$\pm$0.042 & 0.434$\pm$0.046 \\
	       & Ours (RI) & 0.065$\pm$0.037 & 0.062$\pm$0.051 & Ours (RI) & 0.384$\pm$0.032 & {\bf 0.453}$\pm$0.037 \\
	       & Ours (RM) & {\bf 0.018}$\pm$0.010 & {\bf 0.022}$\pm$0.011 & Ours (RM) & {\bf 0.426}$\pm$0.036 & 0.440$\pm$0.037 \\ 
		   \bottomrule
		\end{tabular}
\end{adjustbox}
 	\captionof{table}{\small Results for unsupervised accuracy estimation and error detection. For typical DNN, We use CNN-BN for Digits, ResNet50 for Office-31, ResNet34 for CIFAR10-C, ResNet50 for iWildCam, and Fully Connected Network for Amazon Review. We show the mean and standard deviation of absolute estimation error and F1 score (mean$\pm$std). The numbers are calculated over the training-test dataset pairs in each dataset category. {\bf Bold} numbers are the superior results. } 
 	\label{tab:main-results}
 	\vspace{-0.5cm}
\end{table}

\mypara{Performance for Accuracy Estimation.}
The results in Table~\ref{tab:main-results} show that our method (with $\ens_\textrm{RM}$) achieves significantly better results across various dataset categories compared to existing methods. Specifically, on Digits and CIFAR10-C, it outperforms current state-of-the-art method Proxy Risk significantly (e.g. reduce the error by $>40$\%). On Office-31 and Amazon Review, it also outperforms the others and has a large advantage on the pre-trained models using DANN-arch. 

We would like to emphasize the results on the challenging dataset iWildCam. Our method (with either $\ens_\textrm{RI}$ or $\ens_\textrm{RM}$) outperforms the other methods significantly (e.g., reduce the error by $>70$\%), and the instantiation with $\ens_\textrm{RI}$ gets the best results. 
Note that on iWildCam, the label distribution of the test data is imbalanced and different from that of the training data. In such a case, the representation matching technique will fail since the representations of the two domains may be misaligned.
Thus, the performance of Proxy Risk becomes worse, as it relies on representation matching. Our method with $\ens_\textrm{RM}$ also uses DANN in the ensemble method, but performs significantly better than Proxy Risk, since the diversity helps satisfy condition \textbf{(C)} though the representation matching fails to improve accuracy there. This shows that our ensemble and self-training techniques could alleviate the drawbacks of representation matching in such cases. Furthermore, our method with $\ens_\textrm{RI}$ achieves even better results, which demonstrates the flexibility and effectiveness of our framework. 


\newcommand{\scalefactor}{0.9}
\begin{figure*}[h!]
    \centering
    \begin{subfigure}{\scalefactor\linewidth}
	    \centering
		\includegraphics[width=\linewidth]{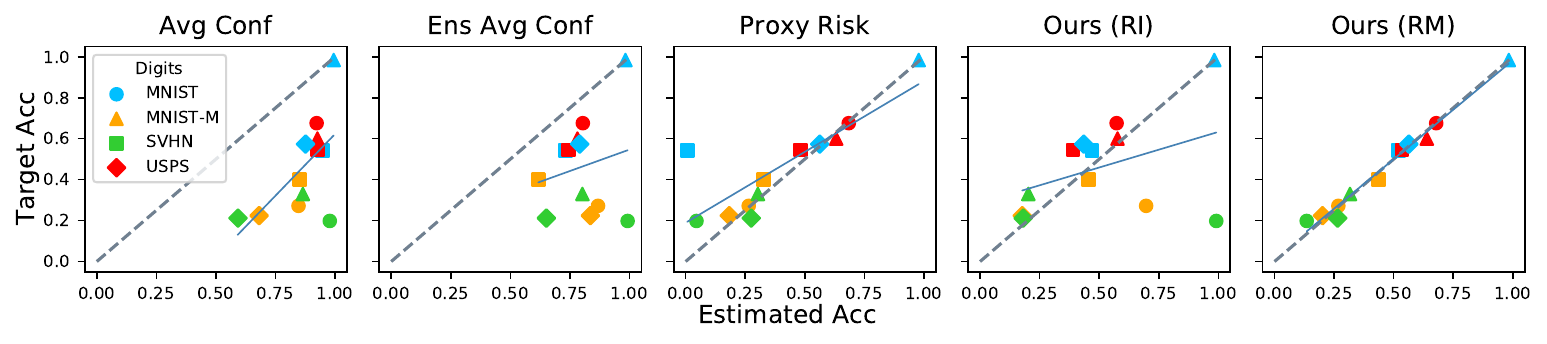}
	\end{subfigure}
    \begin{subfigure}{\scalefactor\linewidth}
		\centering
		\includegraphics[width=\linewidth]{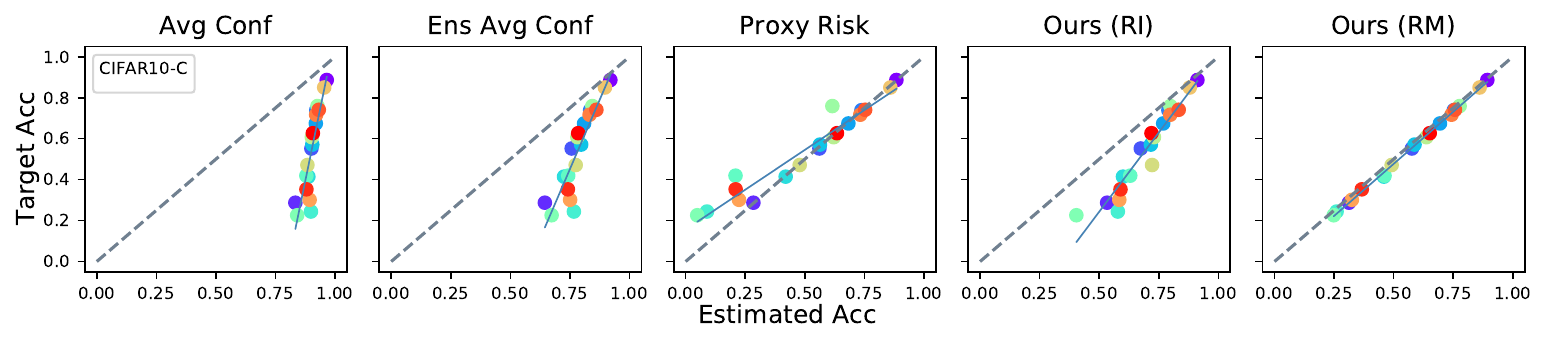}
	\end{subfigure} 
	\caption{\small Accuracy estimation detailed results for each dataset pair from the dataset category Digits and CIFAR10-C (results for other categories are in the Appendix~\ref{sec:plot-accuracy-estimation}). We use typical DNN as the architecture for the model $f$. We use symbols to represent training datasets and colors to represent test datasets. For CIFAR10-C, there is only one training dataset with multiple test datasets. The dashed line represents perfect prediction (target accuracy = estimated accuracy). Points beneath (above) the dashed line indicate overestimation (underestimation). The solid lines are regression lines of the results. 
	}
	\label{fig:accuracy-estimation-main}
	\vspace{-0.3cm}
\end{figure*}

The results for each dataset pair from Digits and CIFAR10-C are plotted in Figure~\ref{fig:accuracy-estimation-main}. The plots show that proxy risk tends to underestimate the accuracy. This is because proxy risk maximizes disagreement which can overly suppress the accuracy. While the average confidence methods tend to overestimate the accuracy, because the model $f$ tends to be overconfident on the data with dataset shift, even when this issue gets rectified in the ensemble average confidence method. 
In comparison, our method with $\ens_\textrm{RM}$ exhibits a clear advantage. The full results for all dataset categories and pre-trained models in Appendix~\ref{sec:plot-accuracy-estimation} show a similar trend.

\mypara{Performance for Error Detection.}
Table~\ref{tab:main-results} shows that our method with $\ens_\textrm{RM}$ outperforms existing methods on all dataset categories. Specifically, our method improves the F1 score by at least 4.4\% on Digits, by at least 25.6\% on Office-31, by at least 2.0\% on CIFAR10-C, by at least 4.7\% on iWildCam and by at least 1.4\% on Amazon Review. This shows the advantages of our method to identify error points in the unlabeled test dataset.  


\mypara{Ablation Studies.}
To study the effect of using ensembles, we vary the ensemble size $N$ ($N=1$ means one single $h$, without ensemble) in our method with $\ens_\textrm{RM}$. Similarly, for self-training, we vary the self-training iteration number $T$ ($T=1$ means no self-training).  Figure~\ref{fig:ablation-study} shows their effect on the average F1 score: both ensemble and self-training techniques have positive effects for identifying error points and thus also for accuracy estimation. Moreover, increasing $T$ and $N$ can lead to further improvement. In addition to $N$ and $T$, we similarly exam the effect of the last hyper-parameter $\gamma$. The figure shows that a wide range of $\gamma$ can lead to good results, so it is easy to pick a suitable $\gamma$.

\newcommand{\scalefactortwo}{0.25}
\begin{figure}[t!]
    \centering
    \begin{subfigure}{\scalefactortwo\linewidth}
	    \centering
		\includegraphics[width=\linewidth]{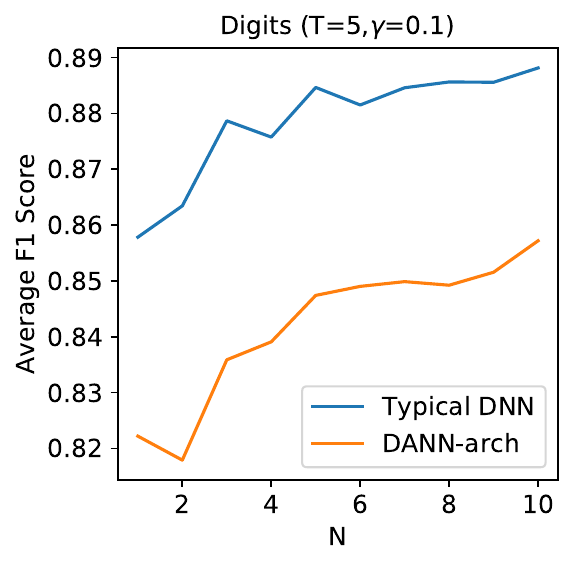}
	\end{subfigure}
	\begin{subfigure}{\scalefactortwo\linewidth}
		\centering
		\includegraphics[width=\linewidth]{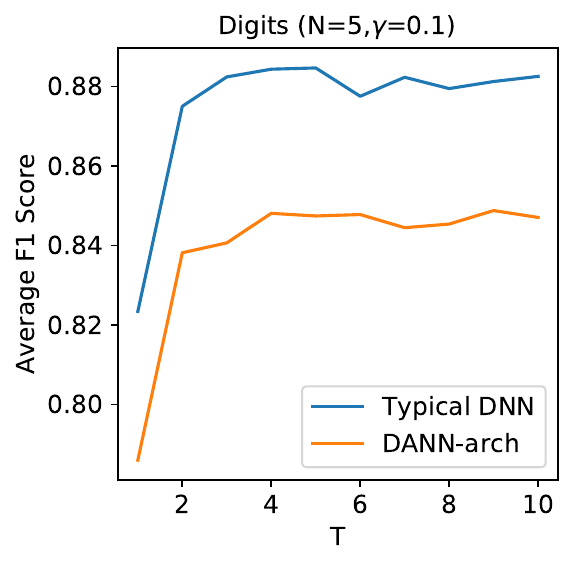}
	\end{subfigure} 
    \begin{subfigure}{\scalefactortwo\linewidth}
		\centering
		\includegraphics[width=\linewidth]{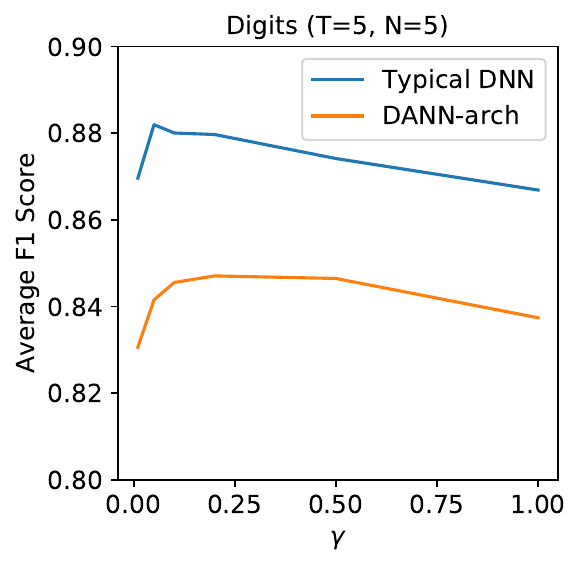}
	\end{subfigure} 
	\caption{\small Ablation study for the effect of ensemble and self-training techniques on Digits (an additional ablation study on CIFAR10-C is in Appendix~\ref{sec:extra-ablation-studies}). $N$ is the number of models in the ensemble, $T$ is the number of self-training iterations, and $\gamma$ is the weighting parameter for the loss term on the pseudo-labeled data. The ensemble training algorithm we use is $\ens_\textrm{RM}$. 
	}
	\label{fig:ablation-study}
	\vspace{-0.3cm}
\end{figure}


\mypara{Validating the Theoretical Analysis.}
Our analysis relies on three conditions \textbf{(A)(B)(C)} stated in Section~\ref{sec:analysis}.  
Our experiments in Appendix~\ref{app:validate-theory} show that empirically they are roughly satisfied. For example, for typical DNN $f$ on MNIST$\rightarrow$MNIST-M, $\tilde{\nu}=3.39\%$, $\tilde{\gamma} = 0.73\%$ while $\sigma_L^2 = 26.58\%$. We note that our theoretical analysis is for formalizing our intuition and is for the worst case. Even when our assumptions are not fully satisfied on some complex datasets, our method can still have empirical performance better than the error bound.

\mypara{Evaluating $f$ with Various Architectures. } We evaluate different architectures for the model $f$ on Digits and show the results in Appendix~\ref{app:evaluate-various-arch}. The results demonstrate that our method consistently outperforms other methods on pre-trained models with various deep learning model architectures.   

\mypara{Analysis on Target Accuracy. } In Appendix~\ref{sec:analysis-target-acc}, we show that the target accuracy of our ensemble check models can be higher or lower than that of the model $f$ depending on the datasets and in both cases, our method can achieve good performance. 

\mypara{Analysis on Proxy Risk. } In Appendix~\ref{sec:analysis-proxy-risk}, we show that the disagreement maximization is crucial for Proxy Risk and combining Proxy Risk with ensemble doesn't lead to better results than our method.



\section{Discussions}
\label{sec:discussion}

While our framework is general and flexible when combined with different ensemble methods, and it is easy to design ensemble methods satisfying the success conditions, we note that some prior knowledge about the data is needed to achieve the best performance. For example, iWildCam has imbalanced classes which hurt representation matching and thus $\mathcal{T}_{\textrm{RM}}$, while $\mathcal{T}_{\textrm{RI}}$ is more suitable for such data.  
Different instantiations thus have their own limitations. 
For $\mathcal{T}_{\textrm{RM}}$, matching failure can cause errors on $f$'s correct points, and too strong matching can decrease the diversity since the models are too restricted.
For $\mathcal{T}_{\textrm{RI}}$, it does not attempt to improve the accuracy on the test inputs and relies only on diversity to satisfy condition \textbf{(C)}. Then it can have worse performance than $\mathcal{T}_{\textrm{RM}}$ on test data with connection to the training data that can be exploited.
The success conditions we identified can guide the design of different instantiations. In this paper we focus on ensembles of deep learning models, since they readily satisfy the conditions. However, other ensemble methods with other types of models may also be useful. We leave the exploration for future work.

\section*{Acknowledgments}
The authors would like to thank Dr. Ankur Taly for his valuable comments. The work is partially supported by Air Force Grant FA9550-18-1-0166, the National Science Foundation (NSF) Grants CCF-FMitF-1836978, IIS-2008559, SaTC-Frontiers-1804648, CCF-2046710 and CCF-1652140, and ARO grant number W911NF-17-1-0405. Jiefeng Chen and Somesh Jha are partially supported by the DARPA-GARD problem under agreement number 885000. 

	\bibliographystyle{plain}
	\bibliography{ref}

\section*{Checklist}


\begin{enumerate}

\item For all authors...
\begin{enumerate}
  \item Do the main claims made in the abstract and introduction accurately reflect the paper's contributions and scope?
    \answerYes{}
  \item Did you describe the limitations of your work?
    \answerYes{See Section~\ref{sec:discussion}.}
  \item Did you discuss any potential negative societal impacts of your work?
    \answerNA{}
  \item Have you read the ethics review guidelines and ensured that your paper conforms to them?
    \answerYes{}
\end{enumerate}

\item If you are including theoretical results...
\begin{enumerate}
  \item Did you state the full set of assumptions of all theoretical results?
    \answerYes{See Section~\ref{sec:analysis} and Appendix~\ref{app:additional_analysis}.}
	\item Did you include complete proofs of all theoretical results?
    \answerYes{See Appendix~\ref{app:additional_analysis}.}
\end{enumerate}

\item If you ran experiments...
\begin{enumerate}
  \item Did you include the code, data, and instructions needed to reproduce the main experimental results (either in the supplemental material or as a URL)?
    \answerYes{See the supplemental material. }
  \item Did you specify all the training details (e.g., data splits, hyperparameters, how they were chosen)?
    \answerYes{See Appendix~\ref{sec:detail-experiment-setup}. }
	\item Did you report error bars (e.g., with respect to the random seed after running experiments multiple times)?
    \answerYes{See Appendix~\ref{sec:multiple-runs}. }
	\item Did you include the total amount of compute and the type of resources used (e.g., type of GPUs, internal cluster, or cloud provider)?
    \answerYes{See Appendix~\ref{sec:detail-experiment-setup}. }
\end{enumerate}

\item If you are using existing assets (e.g., code, data, models) or curating/releasing new assets...
\begin{enumerate}
  \item If your work uses existing assets, did you cite the creators?
    \answerYes{See References. }
  \item Did you mention the license of the assets?
    \answerNA{}
  \item Did you include any new assets either in the supplemental material or as a URL?
    \answerNA{}
  \item Did you discuss whether and how consent was obtained from people whose data you're using/curating?
    \answerNA{}
  \item Did you discuss whether the data you are using/curating contains personally identifiable information or offensive content?
    \answerNA{}
\end{enumerate}

\item If you used crowdsourcing or conducted research with human subjects...
\begin{enumerate}
  \item Did you include the full text of instructions given to participants and screenshots, if applicable?
    \answerNA{}
  \item Did you describe any potential participant risks, with links to Institutional Review Board (IRB) approvals, if applicable?
    \answerNA{}
  \item Did you include the estimated hourly wage paid to participants and the total amount spent on participant compensation?
    \answerNA{}
\end{enumerate}

\end{enumerate}

\newpage
\appendix
\begin{center}
	\textbf{\LARGE Supplementary Material}
\end{center}

 \begin{center}
	\textbf{\large Detecting Errors and Estimating Accuracy on Unlabeled Data with Self-training Ensembles}
\end{center}

In Section~\ref{app:further-discussion}, we have some further discussions to clarify some important points about our work. In Section~\ref{app:additional_analysis}, we present our theoretical results, their proofs and the discussions on the benefit of using ensembles. In Section~\ref{sec:experiment-details}, we describe the detailed settings for the experiments and also present some additional experimental results. 

\section{Further Discussions} 
\label{app:further-discussion}

\subsection{Comparisons with Proxy Risk}
\label{sec:comparison-with-proxy-risk}

Our work and proxy risk work have the following similarities: 
\begin{enumerate}
    \item Both proxy risk and our framework train check models to estimate the accuracy of the pre-trained model $f$ on the unlabeled test dataset $U$ and also identify misclassified points in $U$;
    \item Both proxy risk and one instance of our framework (Algorithm~\ref{alg:ensemble_rm}) use domain-invariant representations (DIR) to improve the accuracy of the check models on the target domain. 
\end{enumerate}

However, they also have the following differences:
\begin{enumerate}
    \item The key ideas are different. The idea of Proxy Risk is to find a check model with maximum disagreement in the good hypothesis class (the set of hypotheses that achieve small DIR loss). Our idea is to {\bf increase} the disagreement on the {\bf mis-classified points} in each iteration of self-training, and mis-classified data points are identified by either accurate prediction or diversity using ensemble;
    \item The training objectives are different: 1) Proxy Risk's objective is applied on the whole unlabeled test set $U$. Ours is only on the selected subset of $U$ (i.e., the currently identified mis-classified points); 2) The terms in the objective to encourage disagreement are different. Proxy risk tries to maximize the disagreement between the model $f$ and the check model $h$ directly on the entire unlabeled test set $U$ while maintaining a small DIR loss (corresponding to the term $-\mathbb{E}_{x \in U_X}\ell(f(x),h'(x))$ in their objective). In contrast, our method encourages disagreement via fitting the check models to the pseudo-labeled dataset $R$ (corresponding to the term $\mathbb{E}_{(x,y)\in R}\ell(h(x),y)$ in our objective). The two terms $\mathbb{E}_{(x,y)\in R}\ell(h(x),y)$ and $-\mathbb{E}_{x \in U_X}\ell(f(x),h'(x))$ are different. For multi-class classification, disagreeing with the pre-trained model is not equivalent to agreeing with the pseudo-labels. This is because the check model can predict some labels different from both the pre-trained model's prediction $f(x)$ and the pseudo-label $\tilde{y}_x$. For example, suppose there are 3 classes, and suppose the pre-trained model predicts $f(x)$= class 0,  the pseudo-label $\tilde{y}_x$=class 1. Our term is ensuring $h(x)$ to be class 1, while their term is ensuring $h'(x)$ to be either class 1 or class 2. Thus, our objective is more specific, and this leads to increasing disagreement and also potentially better prediction from the ensemble, which then leads to the success of the self-training; 
    \item The implementation of the objectives are also different. The proxy risk method uses $L_2$ norm:  $-\mathbb{E}_{x \sim U_X} [\ell(f(x), h'(x))]=\mathbb{E}_{x \sim U_X} [ - \| \bar{h'}(x) - \bar{f}(x) \|_{2} ]$ \footnote{See their code: \url{https://github.com/chingyaoc/estimating-generalization}} while our method uses cross entropy loss $\mathbb{E}_{(x,y)\in R}[\ell(h(x),y)]=\mathbb{E}_{(x,y)\in R} [-\log \bar{h}(x)_{y}]$. Here, $\bar{h'}$ is the softmax output of $h$ (similarly for $\bar{f}$ and $\bar{h}$).
\end{enumerate}

\subsection{Comparisons with Standard Self-training and Ensemble}
\label{sec:comparison-with-standard-self-training-ensemble}

Our self-training ensembles method is different from the standard self-training and ensemble. The differences are:
\begin{enumerate}
    \item The goals are different. Standard self-training+ensemble aims to get accurate predictions. Ours is to increase disagreement on mis-classified points by either getting accurate predictions or getting diversity. 
    \item The techniques are also different, due to the different goals. (1) For each identified mis-classified point, we assign a pseudo-label that is different from the prediction of the pre-trained model, but we do not need the pseudo-labels to be correct, as our goal is disagreement. In contrast, standard methods typically hope the pseudo-labels are correct. (2) We only assign pseudo-labels to the currently identified set of mis-classified points, while standard methods typically assign pseudo-labels for all unlabeled points. This is because we want disagreement on mis-classified points rather than on all points, and we would like to assign pseudo-labels to be different from the prediction of the pre-trained model on mis-classified points.
\end{enumerate}

\section{Complete Proofs and Discussions for Section~\ref{sec:analysis}}
\label{app:additional_analysis}
 
We first present the proof for our main theorem, and then provide some more discussion on the benefit of using ensembles.

\subsection{Provable Guarantees of the Framework} \label{app:relaxed_analysis}

We first prove a technical lemma for constructing $R$ and then use it to prove the theorem. 

First recall the key notions. 

Suppose on points where $f$ is correct, the ensemble models are also approximately correct.   
Let $\nu$ denote the average probability of $h \sim \ens$ making error on the test points where the pre-trained model $f$ is correct:
\begin{align}
    \nu := \Pr_{(x,y_x) \sim U, h \sim \ens} \ [h(x) \neq y_x | f(x) = y_x].
\end{align} 

Suppose the ensemble training with regularization on the pseudo-labeled data $R$ will make the ensemble models disagree with $f$ on $R_X$. Let $\gamma$ denote the average probability of $h \sim \ens$ agreeing with $f$ on the test points in $R_X$:
\begin{align}
    \gamma := & \Pr_{x \sim  R_X, h \sim \ens} \quad \{ h(x) = f(x) \}. 
\end{align}

Suppose $G_X$ is the points in $W_X \setminus R_X$ on which the ensemble agree with the true label $y_x$ with more than $1-\nu$ probability, i.e.,
\begin{align}
    G_X := \{ x \in W_X \setminus R_X: \Pr_{h \sim \ens}\{ h(x) = y_x \} \ge 1- \nu  \}.
\end{align} 
That is, $G_X$ are the points where the ensemble will have correct prediction with high confidence.
We would like the ensemble to have large diversity on the remaining points in $W_X \setminus R_X$. 
Define the diversity there to be 
\begin{align}
    B_X & := W_X \setminus (R_X \cup G_X), 
    \\
    \sigma^2 & := \mathbb{E}[\sigma_x^2 | x \in B_X].
\end{align}

\begin{lemma} \label{lem:constructR} 
Define 
\begin{align}
    B_\eta := \min\{\sigma^2, 1- \nu^2 \}.
\end{align}
For any $\eta \in (0, B_\eta)$, let $\tau = \sqrt{1 - \eta}$, 
and
\begin{align}
    R'_X :=& \left\{ x \in U_X: \Pr_{h \sim \ens}\{ h(x) = f(x) \} <\tau  \right\}.
\end{align}
Then we have 
\begin{align*} 
   & |R'_X \cap (U_X \setminus W_X)| \le  \frac{\nu}{1 - \tau} |U_X \setminus W_X|, \textrm{~~and~~} 
   \\
   & (1-\frac{\gamma}{\tau})|R_X| \leq |R_X \cap R'_X|, G_X \subseteq R'_X, |R'_X \cap B_X|  \ge \frac{\sigma^2 - \eta}{1 - \eta}|B_X|.
\end{align*}
\end{lemma}

\begin{proof}
Consider $x \in U_X - W_X$. We have
\begin{align}
    & \quad \Pr_{(x,y_x) \sim U}\left\{ \Pr_{h \sim \ens}\{ h(x) = f(x)\} < \tau | f(x) = y_x \right \} 
    \\
    & = \Pr_{(x,y_x) \sim U}\left\{ \Pr_{h \sim \ens}\{ h(x) \neq y_x \} \ge 1-\tau | f(x) = y_x \right\} 
    \\
    & \leq \frac{\nu}{1 - \tau}.
\end{align}
So only $\frac{\nu}{1 - \tau}$ fraction of the data points in $U_X \setminus W_X$ will be put into $R_X$, proving the first statement.

Now, consider $x\in W_X$. 
For $x \in R_X$, we have  
\begin{align}
    & \quad \Pr_{x \sim R_X}\left\{ \Pr_{h \sim \ens}\{ h(x) = f(x)\} \geq \tau \right \} \leq \frac{\gamma}{\tau}.
\end{align}
so more than $1-\frac{\gamma}{\tau}$ fraction of the data points in $R_X$ will be put into $R'_X$. 

For $x \in G_X$, since $\eta < 1 - \nu^2$, we have $\nu < \tau$. Note that $f(x) \neq y_x$, thus $x$ will be put into $R'_X$.  
In the following we consider $x\in B_X = W_X \setminus (R_X \cup G_X)$. 

We first show that a significant fraction of $B_X$ has variance larger than $\eta$. 
Since $\sigma^2_x \in [0,1]$, 
\begin{align}
    \sigma^2 &  = \mathbb{E}[\sigma_x^2 | x \in B_X  ]
    \\
    & \le \Pr\{\sigma^2_x \le \eta | x \in B_X \} \cdot \eta + \Pr\{\sigma^2_x > \eta | x \in B_X \} \cdot 1
    \\
    & = (1 - \Pr\{\sigma^2_x > \eta | x \in B_X \}) \cdot \eta + \Pr\{\sigma^2_x > \eta | x \in B_X \} 
\end{align}
leading to
\begin{align}
    \Pr\{\sigma^2_x > \eta | x \in B_X \} \ge \frac{\sigma^2 - \eta}{1 -\eta}.
\end{align}

Now it is sufficient to show that any $x \in W_X$ with $\sigma^2_x > \eta$ will have a small agreement rate $\Pr_{h \sim \ens}\{h(x) = f(x)\}$. 
\begin{align}
    \Pr_{h \sim \ens}[h(x) = f(x)]  
    & \le  \sqrt{\sum_{y \in \mathcal{Y}} (\Pr_{h \sim \ens}[h(x) = y])^2} 
    \\
    & = \sqrt{\Pr_{h_1, h_2 \sim \ens}[h_1(x) = h_2(x)]} 
    \\
    & = \sqrt{ 1 - \sigma^2_x} 
    \\
    & < \tau = \sqrt{ 1 - \eta}. 
\end{align} 
So any point $x \in B_X$ with $\sigma^2_x > \eta$ will fall into $R'_X$, completing the proof.
\end{proof}

Using the above lemma, we arrive at our main theorem.

\begin{theorem}[Restatement of Theorem~\ref{thm:main}] 
Assume in each iteration of the framework, $\tau = \sqrt{1 -\eta}$ for some $\eta \in (0, 3 B_\eta/4)$ where $B_\eta := \min\{\sigma^2, 1- \nu^2 \}$.
Let $\sigma_L^2 > 0$ be a lower bound on the diversity $\sigma^2$, $\tilde{\gamma}$ be an upper bound on $\gamma$, and $\tilde{\nu}$ be an upper bound on $\nu$ over all iterations. Then for any $\delta \in (0, \sigma_L^2/4)$, after at most $\lceil 1/\delta \rceil$ iterations, we can get $\frac{|U_X \setminus R_X|}{|U_X|}$ approximates the accuracy $\acc(f)$ and $R_X$ approximates the mis-classified points $W_X$ as follows: 
\begin{align}
    \left| \acc(f) -  \frac{|U_X \setminus R_X|}{|U_X|}\right| & \le  \max\{\frac{\tilde{\nu}}{1-\tau} (1-e_f), \epsilon e_f\}, \textrm{~where~} \epsilon := \frac{\frac{\tilde{\gamma}}{\tau}\left( 1 + \frac{\tilde{\nu}}{1-\tau} \frac{1-\err_f}{\err_f}\right)}{\frac{\sigma_L^2}{4} - \delta + \frac{\tilde{\gamma}}{\tau}},
    \\
    |W_X \triangle R_X |  & \le \frac{\tilde{\nu}}{1-\tau} |U_X \setminus W_X| + \epsilon |W_X|.
\end{align}
\end{theorem}

\begin{proof}
For each iteration, Lemma~\ref{lem:constructR} implies that the new constructed set $R'_X$ satisfies: 
\begin{align*} 
  & |R'_X \cap (U_X \setminus W_X)| \le \frac{\nu}{1 - \tau} |U_X \setminus W_X|, \textrm{~~and~~} \\
  & (1-\frac{\gamma}{\tau})|R_X| \leq |R_X \cap R'_X|, G_X \subseteq R'_X, |R'_X \cap B_X|  \ge \frac{\sigma^2 - \eta}{1 - \eta}|B_X|.
\end{align*}
Since $\eta \le 3\sigma^2/4$, 
\begin{align} 
\frac{\sigma^2 - \eta}{1 - \eta}  \ge \frac{\sigma^2}{4}.
\end{align} 
Therefore, 
\begin{align} 
    | R'_X \cap B_X|  &  \ge  \frac{\sigma^2}{4} |B_X|.
\end{align}

Suppose $t^*$ is the first iteration when less than $\epsilon$ fraction of $W_X$ is outside $R_X$. Then in any iteration before $t^*$, the newly constructed pseudo-labeled set $R'_X$ loses at most $\frac{\tilde{\gamma}}{\tau}$ fraction of $R_X$, and obtains at least $\frac{\sigma_L^2}{4}$ fraction of $B_X \cap G_X = W_X \setminus R_X$. Since more than $\epsilon$ fraction of $W_X$ is outside $R_X$, it can then be verified that 
\begin{align}
    \frac{\sigma_L^2}{4} |W_X \setminus R_X| - \frac{\tilde{\gamma}}{\tau} |R_X| > \delta |W_X \setminus R_X|.
\end{align}
Therefore, after each iteration, the framework adds more than $\delta$ fraction of $|W_X \setminus R_X|$ in $R_X$. This can happen at most $1/\delta$ iterations, so $t^* \le \lceil 1/\delta \rceil$. 

Now consider the last iteration, and apply Lemma~\ref{lem:constructR}, then 
\begin{align} \label{eq:rx-wx}
    |R_X \setminus W_X| = |R_X \cap (U_X \setminus W_X)| \le \frac{\nu}{1 - \tau} |U_X \setminus W_X| \le \frac{\tilde{\nu}}{1-\tau} |U_X \setminus W_X|. 
\end{align} 
Equation (\ref{eq:rx-wx}) together with the fact that there are less than $\epsilon$ fraction of $W_X$ outside $R_X$ lead to the two statements in the theorem.   
\end{proof}

By setting the $\eta = 7/16$ and $\delta = 4\tilde{\gamma}/3$, we have the following corollary as an example. 

\begin{corollary} \label{cor:main_sim}
Assume in each iteration of the framework, $\nu < 1/2,  \sigma^2 > 7/12$, and $\tau = 3/4$ where $B_\eta := \min\{\sigma^2, 1- \nu^2\}$.
Let $\sigma_L^2 > 0$ be a lower bound on the diversity $\sigma^2$, $\tilde{\gamma}$ be an upper bound on $\gamma$, and $\tilde{\nu}$ be an upper bound on $\nu$ over all iterations. If $\sigma_L^2 \ge \frac{16\tilde{\gamma}}{3}$, then after at most $\lceil 3/(4\tilde{\gamma}) \rceil$ iterations, we can get $\frac{|U_X \setminus R_X|}{|U_X|}$ approximates the accuracy $\acc(f)$ and $R_X$ approximates the mis-classified points $W_X$ as follows: 
\begin{align}
    \left| \acc(f) -  \frac{|U_X \setminus R_X|}{|U_X|}\right| & \le  \max\{4\tilde{\nu}(1- \err_f), \epsilon \err_f\}, \textrm{~where~} \epsilon := \frac{16\tilde{\gamma}}{3 \sigma_L^2}\left( 1 + 4\tilde{\nu} \frac{1-\err_f}{\err_f}\right),
    \\
    | W_X \triangle R_X |  & \le 4\tilde{\nu}  |U_X \setminus W_X| + \epsilon |W_X|.
\end{align}
\end{corollary}

\begin{proof}
In Theorem~\ref{thm:main}, note that $\eta = 7/16$ leads to $\tau = 3/4$.
The bounds on $\nu$, $\gamma$, and $\sigma^2$ comes from the requirement that $3 B_\eta/4 > 7/16$ so that there exists such an $\eta \in (7/16, 3 B_\eta/4)$.
\end{proof}

\subsection{Discussion on Using Ensembles}

Our analysis of the framework clearly relies on the effect of self-training. It also shows the benefit of the ensemble: the diversity (combined with low errors of the ensemble on points correctly classified by $f$) allows to identify mis-classified points. 

Here we present more discussion on the approach of using ensembles to estimate the accuracy and to provide further insight into their benefit compared to some other existing approaches. For simplicity, we analyze binary classification with $\mathcal{Y} = \{0, 1\}$ in this section unless stated otherwise. We provide an exact characterization of the estimation error (i.e., how far the agreement rate is from the actual accuracy). It implies that to get a good estimation, one should use ensembles with small prediction errors on the test points. More importantly, the estimation can be further improved if the ensemble's prediction has proper correlation with $f$, which then shows the advantage of an ensemble of models instead of a single model.  

Let $\err_{h,x}$ be the indicator that $h$ mis-classifies $x$, $\err_\ens$ be the expected error of $h$ on $U$ (over the distribution of $h$), and $\err_f$ be the error of $f$:
\begin{align*}
    \err_{h,x} := \mathbb{I}[h(x) \neq y_x], \quad 
    \err_\ens := \mathbb{E}_{h,x} [\err_{h,x}],  \quad
    \err_f := \mathbb{E}_{x} [\err_{f,x}].
\end{align*}
Let $\ar_x(f, \ens)$ be the agreement rate between $f$ and $h$'s on a point $x$, and  $\ar(f, \ens)$ be that on the whole test set:
\begin{align*}
    \ar_x(f, \ens) & := \Pr_{h \sim \ens}\{ h(x) = f(x)\}, \\
    \ar(f, \ens) & := \mathbb{E}_{x \in U_X} [\ar_x(f, \ens)].
\end{align*}
Recall that we are using $\ar(f, \ens)$ as an estimate of the accuracy of $f$ on $U$.

\begin{lemma} \label{lem:ensemble}
For binary classification,
\begin{align}
    \acc(f) - \ar(f, \ens)
    & = \err_\ens (1 - 2 \err_f) - 2  \cov(\err_{f,x}, \err_{h,x})
\end{align}
where $\cov(\err_{f,x}, \err_{h,x})$ is the covariance between $\err_{f,x}$ and $\err_{h,x}$.
For multi-class classification,
\begin{align}
    \err_\ens (1 - 2 \err_f) - 2  \cov(\err_{f,x}, \err_{h,x}) \le \acc(f) - \ar(f, \ens)
    \le \err_\ens (1 -  \err_f) - \cov(\err_{f,x}, \err_{h,x}).
\end{align}
\end{lemma}
 
\begin{proof}
We have 
\begin{align}
    \acc(f) - \ar(f, \ens)
    & = \mathbb{E} \{\mathbb{I}[f(x) = y_x]\} -  \mathbb{E} \{ \mathbb{I}[f(x) = h(x)] \}
    \\
    & = \mathbb{E} \left\{ \mathbb{I}[f(x) = y_x] -   \mathbb{I}[f(x) = h(x)] \right\}
    \\
    & = \mathbb{E} \left\{    \mathbb{I}[f(x) \neq h(x)] - \mathbb{I}[f(x) \neq y_x]  \right\}.
\end{align}
The first term can be decomposed into two parts: 
\begin{align}
    \mathbb{I}[f(x) \neq h(x)] 
    & = \mathbb{I}[f(x) \neq h(x), f(x) = y_x] + \mathbb{I}[f(x) \neq h(x), f(x) \neq y_x] 
\end{align}
and the two parts can be transformed as:
\begin{align}
    \mathbb{I}[f(x) \neq h(x), f(x) = y_x]
    & =  \mathbb{I}[h(x) \neq y_x, f(x) = y_x]  
    \\
    & = \err_{h,x} (1 - \err_{f,x}),
    \\
    \mathbb{I}[f(x) \neq h(x), f(x) \neq y_x]  
    & = \mathbb{I}[h(x) = y_x, f(x) \neq y_x] \label{eqn:ensemble}
    \\
    & = e_{f,x} ( 1 - e_{h,x})
\end{align}
where the second to last line follows from that in binary classification, $f(x) \neq h(x)$ and $f(x) \neq y_x$ is equivalent to $ h(x) = y_x$ and $ f(x) \neq y_x$.
Therefore,
\begin{align}
    \acc(f) - \ar(f, \ens)  
    & = \err_\ens - 2 \mathbb{E} [\err_{h,x} \err_{f,x}]
    \\
    & = \err_\ens - 2 (\err_\ens \err_f + \cov(\err_{h,x}, \err_{f,x})).
\end{align}
Rearranging the terms completes the proof.

For multi-class, we can replace (\ref{eqn:ensemble}) by the bounds:
\begin{align}
    \mathbb{I}[h(x) = y_x, f(x) \neq y_x] \le \mathbb{I}[f(x) \neq h(x), f(x) \neq y_x]   \le \mathbb{I}[f(x) \neq y_x].
\end{align}
\end{proof}

The bound suggests using $\ens$ with a small prediction error $\err_\ens$. More importantly, the estimation can be improved by a proper correlation between $f$ and the ensemble models: even when the ensemble models don't have very small error $\err_\ens$, they can still lead to a good estimation, as long as they have a proper covariance with $f$. 
More precisely, typically $\err_f < 1/2$, so the covariance should not be negative, but also should not be too positive. For example, when the ensemble models overly agrees with $f$ (e.g., in the extreme case $h(x) = f(x)$ for all $x \in U_X$ and all $h \sim \ens$), it leads to over-estimation of the accuracy, and we should decrease the correlation (more discussion in the next subsection). 

It is also instructive to compare our method to some existing methods.
(1) Our analysis is more general and tighter than that for using a single model in~\cite{chuang2020estimating}. The setting is a special case of ours. More important, our bound is tighter and reveals that an ensemble with proper correlation can improve the estimation, justifying the advantage of an ensemble over a single model.
(2) Our analysis is also more general than the classic notion of calibration. We show that if the ensemble has perfect calibration then the agreement rate equals the accuracy of the pre-trained model. On the other hand, our lemma shows that even without calibration, proper ensembles can still give good estimation. 

Detailed comparisons are presented below. 

\mypara{Comparison with Proxy Risk.}
Recall that the proxy risk method~\cite{chuang2020estimating} is to use invariant representation domain adaptation methods to find the $h$ of maximum disagreement with $f$ on $U_X$. That is, it aims to get the $h \in \mathcal{H}$ such that $\ar(f,h)$ is smallest where $\mathcal{H}$ is the set of hypotheses with small errors on the original training data and small distances between the distributions of the representations of the training and test data, i.e.,
\begin{align}
    \mathcal{H} = \{h \in \mathcal{P}: \textrm{~error of $h$ on the training data~} + \alpha d(p^\phi_S, p^\phi_T) \le \epsilon\}
\end{align}
where $\mathcal{P}$ is the set of networks for domain adaptation, $d(p^\phi_S, p^\phi_T)$ is some distance between the distributions of the representations of the training and the test data, and $\alpha, \epsilon$ are hyperparameters.

The main idea behind the proxy risk method is Lemma 4 in their paper, which states (rephrased to our context): 
\begin{align}
    \left|\sup_{h \in \mathcal{H}} \mathbb{E}_{x \sim U_X} \mathbb{I}[f(x) \neq h(x)] - \err_f \right| \le \sup_{h \in \mathcal{H}} \err_h
\end{align}
where $\err_h$ is the error of $h$ on the test set, i.e., $\err_h = \mathbb{E}_{(x, y_x) \sim U} \{ \mathbb{I}[h(x) \neq y_x] \}$. 

Our bound is more general and tighter. 
We first show that their bound can be recovered from ours. More precisely, the proxy risk method is equivalent to using an ensemble method $\ens$ that outputs the $\hat{h} \in \mathcal{H}$ of maximum disagreement with $f$.
Then the output distribution of $\ens$ concentrates on $\hat{h} \in \mathcal{H}$. Our bound then leads to:
\begin{align}
    \left|\sup_{h \in \mathcal{H}} \mathbb{E}_x \mathbb{I}[f(x) \neq h(x)] - \err_f \right| 
    & = \left| \mathbb{E}_x \mathbb{I}[f(x) \neq \hat{h}(x)] - \err_f \right| 
    \\
    & = \bigg| \acc(f) - \mathbb{E}_{h \sim \ens} [\ar(f, h)] \bigg|
    \\
    & = |\err_\ens - 2 \mathbb{E} [\err_{f,x} \err_{h,x}]|
    \\
    & = | \err_{\hat{h}}  - 2 \mathbb{E}_{x} [\err_{f,x} \err_{h,x}]|.
\end{align}
Since $\err_{f,x}$ and $\err_{h,x}$ are in $\{0, 1\}$, it is easy to see that 
\begin{align}
    0 \le \mathbb{E}_{x} [\err_{f,x} \err_{h,x}]  \le \min\{ \mathbb{E}_x [\err_{f,x}], \mathbb{E}_x [\err_{\hat{h},x} ]  \} = \min\{\err_f, \err_{\hat{h}} \}.
\end{align}
Therefore,
\begin{align}
    \left|\sup_{h \in \mathcal{H}} \mathbb{E}_x \mathbb{I}[f(x) \neq h(x)] - \err_f \right|
    & = | \err_{\hat{h}}  - 2 \mathbb{E}_{x} [\err_{f,x} \err_{h,x}]|
    \\
    & \le \err_{\hat{h}} 
    \\
    & \le \sup_{h \in \mathcal{H}} \err_h
\end{align}
recovering the bound in the proxy risk paper. 

The above calculation also shows that our bound is tighter. Their bound is only for the case when only one check model $\hat{h}$ is learned and also for the worst case. First, it is challenging to find an $\hat{h}$ with a small error in many practical scenarios. For example, the test data contains outlier inputs which are not similar to the training data. It is then unlikely to find an $\hat{h}$ with small errors on these data points since no enough label information is available. However, it is still possible to have a good estimation of the accuracy, since the outlier data are different from the training data and thus can be detected, and we know that $f$ is likely to make errors on them. 
Second, the bound is also too pessimistic. In the experiments, we observed that the proxy risk method can still achieve reasonable estimation (about $10\%$ away from the true accuracy), even when the error of $\hat{h}$ is very large (e.g., $> 60\%$ while $\sup_{h \in \mathcal{H}} \err_h$ is even larger).

Our lemma suggests that by allowing an ensemble $h \sim \ens$ with proper diversity, we have more flexibility and can significantly improve the pessimistic bound. 
For the example given above, the ensemble method can potentially handle the outlier input data: for hypotheses agreeing with the training data, they are still likely to have disagreement on the outlier data and this disagreement thus reveals the potential error of $f$ there, leading to an accurate estimation of the accuracy.  
We thus propose to use an ensemble method for estimating the accuracy. 

\mypara{Comparison with Calibration.} 
A classic notion for uncertainty estimation is calibration of the machine learning model. 
It is well-known that if the pre-trained model $f$ outputs confidence scores for class labels and the confidence is well-calibrated, then the average confidence approximates its accuracy. Unfortunately, it has been observed that many machine learning systems, in particular modern neural networks, are poorly calibrated, especially on test data with distribution shift~\citep{guo2017calibration,ovadia2019can} which is the most interesting case for accuracy estimation.

On the other hand, one can hope to obtain an ensemble of models that is well calibrated, such as the deep ensemble method~\citep{lakshminarayanan2017simple}. Below we show that a notion of well-calibration of the ensemble implies the agreement rate between the ensemble and the pre-trained model is a good estimation of the accuracy of the pre-trained model.
Formally, we consider the simplified setting of perfect calibration defined as follows.

\begin{definition}[Perfect Calibration]
An ensemble $\ens$ of classifiers has perfect calibration on the dataset $U = \{(x,y_x)\}$ and function $f$, if for any class label $k \in \mathcal{Y}$ and any $p \in [0,1]$,
\begin{align}
\label{eq:perfect-calibration}
\Pr_{(x, y_x) \sim U} \left[ y_x =k | C_k(x) = p, f(x) = k \right] = p.
\end{align} 
where $C_k(x) := \Pr_{h \sim \ens}[h(x) = k]$ is the confidence score of $\ens$ for label $k$ on the input $x$. 
\end{definition}

(The definition and the later analysis also applies to a classifier $ C_k(x)$ outputting confidence scores, or replacing $U$ with a data distribution.)

\begin{proposition}
If the ensemble $\ens$ has perfect calibration, then $\ar(f, \ens) = \acc(f)$.
\end{proposition}
\begin{proof}
By definition, we have
\begin{align}
    \ar(f, \ens) 
    & = \Pr_{h,x}[h(x) = f(x)]
    \\
    & = \mathbb{E}_x \left[ \Pr_h[h(x) = f(x)] \right]
    \\
    & = \mathbb{E}_x \left[ C_{f(x)}(x) \right]
    \\
    & = \mathbb{E} \left\{ \mathbb{E}\left[ C_{f(x)}(x)  |  C_k(x) = p,  f(x) = k \right] \right\}  
    \\
    & = \mathbb{E} \left\{ \mathbb{E}\left[ p  |  C_k(x) = p,  f(x) = k \right] \right\}  
    \\
    & = \mathbb{E} \left\{ \mathbb{E}\left\{ \Pr\left[ y_x =k | C_k(x) = p, f(x) = k \right]  |  C_k(x) = p,  f(x) = k \right\} \right\}  
    \\
    & = \mathbb{E} \left\{  \Pr\left[ y_x =k | C_k(x) = p, f(x) = k \right]    \right\} 
    \\
    & =  \Pr\left[ y_x = f(x) \right]   
    \\
    & = \acc(f).
\end{align}
The third line follows from the definition of $C_k(x)$, the fourth line from the law of total expectation, the sixth line from perfect calibration, and the eighth line from the law of total expectation.
\end{proof}

\begin{table}[!t]
    \centering 
		\begin{tabular}{c|cccc}
			\toprule 
	       $x$ &  $x_2^-$ & $x_1^-$ & $x_1^+$ & $x_2^+$ 
	       \\
	       $y_x$ & $-$  & $-$ & $+$ & $+$ 
	       \\ \hline
	       $f(x)$ & $-$ & $+$ & $-$ & $+$
	       \\
	       $h(x)$ & $-$ & $+$ & $+$ & $-$
	       \\
			\bottomrule
		\end{tabular} 
	\caption[]{\small An illustrative example showing even if the ensemble is not well calibrated, it is still possible for the agreement rate to be a good estimation of the accuracy. } 
	\label{tab:cal}
\end{table}

On the other hand, our Lemma~\ref{lem:ensemble} is more general: it shows even if the ensemble is not well calibrated, it is still possible for the agreement rate to be a good estimation of the accuracy. 
For illustration, consider a simple example with 4 points in $U$, and only one model $h$ from $\ens$ (if $h(x) = k$, we view it as $\Pr_h[h(x) = k] = 1$). 
The predictions of $f$ and $h$ are shown in Table~\ref{tab:cal}.
It is easy to see that $h$ is not well-calibrated, e.g., 
\[
\Pr_h [ y_x = + | C_+(x) = 1] = \Pr_h [ y_x = + | h(x) = + ] = 1/2 \ll 1.
\]
On the other hand, $\ar(f, \ens) = \acc(f) = 1/2$. From the perspective of Lemma~\ref{lem:ensemble}, although the ensemble has a large error $\err_\ens = 1/2$, its predictions and those of $f$ are properly correlated, such that
\[
\acc(f) - \ar(f, \ens) = \err_\ens - 2 \mathbb{E} [\err_{h,x} \err_{f,x}] = \frac{1}{2} - 2 \cdot \frac{1}{4} = 0
\]
leading to an accurate estimation of the accuracy of $f$.

\section{Experimental Details}
\label{sec:experiment-details}

\subsection{Setup}
\label{sec:detail-experiment-setup}

\subsubsection{Computing Infrastructure}
We run all experiments with PyTorch and NVIDIA GeForce RTX 2080Ti GPUs.   

\subsubsection{Dataset}
In our problem setting, we need a training dataset $\Dtr$ and a test dataset $U$. We evaluate our methods on five dataset categories. Each dataset category contains several evaluation tasks (or several training-test dataset pairs). We introduce each of the dataset categories below. 

\mypara{Digits. } We investigate four digit datasets, which are MNIST~\citep{lecun1998mnist}, MNIST-M~\citep{ganin2016domain}, SVHN~\citep{netzer2011reading} and USPS~\citep{hull1994database}. They all contain digit images with digits from 0 to 9. MNIST contains 60,000 training images and 10,000 test images; MNIST-M contains 59,001 training images and 9,001 test images; SVHN contains 73,257 training images and 26,032 test images; USPS contains 7,291 training images and 2,007 test images. We can construct 12 different training-test dataset pairs from them.  

\mypara{Office-31. } \citep{saenko2010adapting} is the most widely used dataset for visual domain adaptation, with 4,652 images and 31 categories collected from three distinct domains: \textit{Amazon}(A), \textit{Webcam}(W), and \textit{Dslr}(D). \textit{Amazon} contains 2,817 images; \textit{Webcam} contains 795 images; \textit{Dslr} contains 498 images. The images are cropped to be the size of $224\times 224$. We can construct 6 different training-test dataset pairs from them. 

\mypara{CIFAR10-C. } We use CIFAR10~\citep{krizhevsky2009learning} as the training dataset and CIFAR10-C~\citep{hendrycks2019benchmarking} as the test dataset. CIFAR10 contains 50,000 training images and 10,000 test images. CIFAR10-C contains test images with 19 corruption types and 5 severity levels. We only consider severity level of 5. For each corruption type, it contains 10,000 test images generated from CIFAR10 test images by applying the corruption. We have 19 different training-test dataset pairs in total.   

\mypara{iWildCam. } The iWildCam 2020 Competition Dataset~\citep{beery2020iwildcam} contains animal images with 186 species collected by various camera traps. The task is multi-class species classification. We use a variant of it, which is proposed by~\cite{koh2020wilds}. It contains 142,202 training images and 7,861 in-distribution test images from 245 locations. The validation set contains 20,784 images from 32 locations and the test set contains 38,943 images from 47 locations. The locations in the validation set and the test set are different from those in the training set. We create 10 target datasets by sampling data from the validation set and the test set. Each target dataset contains images from 50 locations, which are different from those in the training data. So we have 10 training-test dataset pairs. To reduce computational cost, we resize each image to $64\times 64$.

\mypara{Amazon Review. } The Multi-Domain Sentiment Dataset constructed in~\citep{blitzer7domain} is a dataset for sentiment domain adaptation. It contains Amazon product reviews for four different product types: books, DVDs, electronics and kitchen appliances. For each domain, it has 1,000 positive and 1,000 negative examples. We name it Amazon Review dataset and construct 12 different training-test dataset pairs from it.  

\subsubsection{Implementation and Hyperparameters}
\label{sec:implementation-hyperparam}
 In our algorithms, we need to set hyperparameters $T$, $N$, $\alpha$ and $\gamma$. We specify $\alpha$ in the training configuration (See Section~\ref{sec:model-arch-training-config}). We set $T=5$ for all dataset pairs. We set $N=20$ for Amazon Review dataset and set $N=5$ for other dataset categories. For the implementation of self-training objective, we randomly sample data points from $\Dtr\cup R$ in batches and weight the loss term on the pseudo labeled data by $\gamma$. We set $\gamma=0.1$ for all dataset pairs. For dataset pairs in Digits that use USPS as target dataset, we repeat $R$ ten times and then add it to the training set due to the much smaller size of the test set in USPS compared to the source training set. This is equivalent to multiplying $\gamma$ by 10.

\subsubsection{Model Architecture and Training Configuration}
\label{sec:model-arch-training-config}

We introduce the model architectures and the training configurations for each dataset category below. 

\mypara{Digits. } We use a neural network named CNN-BN, which has two convolutional layers, two full connected layers and batch normalization layers as the typical DNN model. The DANN architecture is adapted from the one used in \citep{chuang2020estimating}. We train all models using Adam optimizer with learning rate of $10^{-3}$ and batch size of 128. For supervised learning, we train the model for 20 epochs. For domain adaptive learning, we train DANN for 100 epochs. We adopt the original progressive training strategy for the discriminator~\citep{ganin2016domain} where the weight $\alpha$ for the domain-invariant loss is initiated at 0 and is gradually changed to 0.1 using the schedule $\alpha=(\frac{2}{1+\text{exp}(-10\cdot p)}-1)\cdot 0.1$, where $p$ is the training progress linearly changing from 0 to 1. When fine-tuning the models, we set $\alpha=0.1$ and use Adam optimizer with learning rate of $10^{-3}$. The model architecture for DANN is presented below. 

\begin{center}
\small
 \begin{tabular}{||c||} 
 \hline
 Encoder  \\ [0.5ex] 
 \hline\hline
 nn.Conv2d(3, 64, kernel$\_$size=5)  \\ 
 \hline
 nn.BatchNorm2d \\
 \hline
 nn.MaxPool2d(2)  \\
 \hline 
 nn.ReLU \\
 \hline
 nn.Conv2d(64, 128, kernel$\_$size=5)  \\ 
 \hline
 nn.BatchNorm2d \\
 \hline
 nn.Dropout2d \\
 \hline
 nn.MaxPool2d(2)  \\
 \hline 
 nn.ReLU \\
 \hline \hline
 nn.Conv2d(128, 128, kernel$\_$size=3, padding=1)  \\
 \hline
 nn.BatchNorm2d \\
 \hline
 nn.ReLU \\
 \hline
 $\times 2$ \\
 \hline
\end{tabular}
\end{center}
\small
\begin{center}
\begin{tabular}{||c||} 
 \hline
 Predictor  \\ [0.5ex] 
 \hline\hline
 nn.Conv2d(128, 128, kernel$\_$size=3, padding=1)  \\
 \hline
 nn.BatchNorm2d \\
 \hline
 nn.ReLU \\
 \hline
 $\times 3$ \\
 \hline\hline
 flatten \\
 \hline
 nn.Linear(2048, 256)  \\ 
 \hline 
 nn.BatchNorm1d \\
 \hline
 nn.ReLU  \\
 \hline
 nn.Linear(256, 10) \\
 \hline 
 nn.Softmax \\
 \hline
\end{tabular}
\quad
 \begin{tabular}{||c||} 
 \hline
 Discriminator  \\ [0.5ex] 

 \hline\hline
 nn.Conv2d(128, 128, kernel$\_$size=3, padding=1)  \\
 \hline
 nn.ReLU \\
 \hline
 $\times 5$ \\
 \hline\hline
 Flatten \\
 \hline
 nn.Linear(2048, 256)  \\ 
 \hline 
 nn.ReLU  \\
 \hline 
 nn.Linear(256, 2) \\
 \hline 
 nn.Softmax \\
 \hline
\end{tabular}
\end{center}

\mypara{Office-31. } We use ResNet50~\citep{he2016deep} as the typical DNN model. For the DANN architecture, we use ResNet50 followed by a four-layer fully connected network with width 256 as the feature extractor.  The main classifier is a three-layer fully connected network with width 256 and the auxiliary classifier is a seven-layer fully connected network with width 256. We train all models for 100 epochs using Adam optimizer with batch size of 32 and learning rate schedule. The initial learning rate is $10^{-5}$ and it decreases to $10^{-6}$ after 50 epochs training. We augment the training data using random resized crop and random horizontal flip. The weight $\alpha$ for DANN training is initiated at 0 and is gradually changed to 1 using the same schedule discussed above. When fine-tuning the models, we set $\alpha=1$ and use Adam optimizer with learning rate of $5\times 10^{-6}$. The model architecture for DANN is presented below. 

\begin{center}
 \begin{tabular}{||c||} 
 \hline
 Encoder  \\ [0.5ex] 
 \hline\hline
 ResNet50(pretrained=True) \\
 \hline 
 nn.Linear(2048, 256)  \\ 
 \hline
 nn.ReLU  \\
 \hline \hline
 nn.Linear(256, 256)  \\
 \hline
 nn.ReLU \\
 \hline
 $\times 4$ \\
 \hline
\end{tabular}
\quad
 \begin{tabular}{||c||} 
 \hline
 Predictor  \\ [0.5ex] 
 \hline\hline
 nn.Linear(256, 256)  \\
 \hline
 nn.BatchNorm1d \\
 \hline
 nn.ReLU  \\
 \hline
 $\times 2$  \\
 \hline \hline
 nn.Linear(256, 31)  \\
 \hline
 nn.Softmax \\
 \hline
\end{tabular}
\quad
\begin{tabular}{||c||} 
 \hline
 Discriminator  \\ [0.5ex] 
 \hline\hline
  nn.Linear(256, 256)  \\ 
 \hline
 nn.ReLU  \\
 \hline
 $\times$6\\
 \hline\hline
 nn.Linear(256, 2)  \\
 \hline
 nn.Softmax \\
 \hline
\end{tabular}
\end{center}    

\mypara{CIFAR10-C. } We use ResNet34~\citep{he2016deep} as the typical DNN model. For the DANN architecture, we use ResNet34 as the feature extractor. The main classifier is a three-layer fully connected network with width 256 and the auxiliary classifier is a seven-layer fully connected network with width 256. We train all models for 100 epochs using Stochastic Gradient Decent (SGD) optimizer with Nesterov momentum and learning rate schedule. We set momentum $0.9$ and $\ell_2$ weight decay with a coefficient of $10^{-4}$. The initial learning rate is $0.1$ and it decreases by $0.1$ at 50, 75 and 90 epoch respectively. The batch size is $128$. We augment the training data using random crop with padding and random horizontal flip. The weight $\alpha$ for DANN training is initiated at 0 and is gradually changed to 0.1 using the same schedule discussed above. When fine-tuning the models, we set $\alpha=0.1$ and use the same SGD optimizer with a learning rate of $10^{-4}$. The model architecture for DANN is presented below. 

\begin{center}
 \begin{tabular}{||c||} 
 \hline
 Encoder  \\ [0.5ex] 
 \hline\hline
 ResNet34 \\
 \hline 
 nn.Linear(512, 256)  \\ 
 \hline
 nn.ReLU  \\\hline
\end{tabular}
\quad
 \begin{tabular}{||c||} 
 \hline
 Predictor  \\ [0.5ex] 
 \hline\hline
 nn.Linear(256, 256)  \\
 \hline
 nn.ReLU  \\
 \hline
 $\times 2$  \\
 \hline \hline
 nn.Linear(256, 10)  \\
 \hline
 nn.Softmax \\
 \hline
\end{tabular}
\quad
\begin{tabular}{||c||} 
 \hline
 Discriminator  \\ [0.5ex] 
 \hline\hline
  nn.Linear(256, 256)  \\ 
 \hline
 nn.ReLU  \\
 \hline
 $\times$6\\
 \hline\hline
 nn.Linear(256, 2)  \\
 \hline
 nn.Softmax \\
 \hline
\end{tabular}
\end{center}

\mypara{iWildCam. } We use ResNet50~\citep{he2016deep} as the typical DNN model. For the DANN architecture, we use ResNet50 followed by a four-layer fully connected network with width 256 as the feature extractor.  The main classifier is a three-layer fully connected network with width 256 and the auxiliary classifier is a seven-layer fully connected network with width 256. We train all models for 50 epochs using Adam optimizer with batch size of 128 and learning rate of $10^{-5}$. The weight $\alpha$ for DANN training is initiated at 0 and is gradually changed to 1 using the same schedule discussed above. When fine-tuning the models, we set $\alpha=1$ and use Adam optimizer with learning rate of $10^{-5}$. The model architecture for DANN is presented below. 

\begin{center}
 \begin{tabular}{||c||} 
 \hline
 Encoder  \\ [0.5ex] 
 \hline\hline
 ResNet50(pretrained=True) \\
 \hline 
 nn.Linear(2048, 256)  \\ 
 \hline
 nn.ReLU  \\
 \hline \hline
 nn.Linear(256, 256)  \\
 \hline
 nn.ReLU \\
 \hline
 $\times 4$ \\
 \hline
\end{tabular}
\quad
 \begin{tabular}{||c||} 
 \hline
 Predictor  \\ [0.5ex] 
 \hline\hline
 nn.Linear(256, 256)  \\ 
 \hline
 nn.BatchNorm1d \\
 \hline
 nn.ReLU  \\
 \hline
 $\times 2$  \\
 \hline \hline
 nn.Linear(256, 186)  \\
 \hline
 nn.Softmax \\
 \hline
\end{tabular}
\quad
\begin{tabular}{||c||} 
 \hline
 Discriminator  \\ [0.5ex] 
 \hline\hline
  nn.Linear(256, 256)  \\ 
 \hline
 nn.ReLU  \\
 \hline
 $\times$6\\
 \hline\hline
 nn.Linear(256, 2)  \\
 \hline
 nn.Softmax \\
 \hline
\end{tabular}
\end{center}     

\mypara{Amazon Review. } We use TF-IDF~\citep{tfidf} to transform texts into feature vectors with dimension 5,000. Then we build fully connected networks with ReLU activation. For typical DNN model, we use a four-layer fully connected network with width 128. For the DANN architecture, we use a four-layer fully connected network with width 128 as the feature extractor. The main classifier is a three-layer fully connected network with width 128 and the auxiliary classifier is a seven-layer fully connected network with width 256. We train all models for 50 epochs using Adam optimizer with batch size of 8 and learning rate of $10^{-3}$. The weight $\alpha$ for DANN training is initiated at 0 and is gradually changed to 1 using the same schedule discussed above. When fine-tuning the models, we set $\alpha=1$ and use Adam optimizer with learning rate of $10^{-3}$. The model architecture for DANN is presented below.     

\begin{center}
 \begin{tabular}{||c||} 
 \hline
 Encoder  \\ [0.5ex] 
 \hline\hline
 nn.Linear(5000, 128)  \\ 
 \hline
 nn.ReLU  \\
 \hline \hline
 nn.Linear(128, 128)  \\
 \hline
 nn.ReLU \\
 \hline
 $\times 3$ \\
 \hline
\end{tabular}
\quad
 \begin{tabular}{||c||} 
 \hline
 Predictor  \\ [0.5ex] 
 \hline\hline
 nn.Linear(128, 128)  \\ 
 \hline
 nn.ReLU  \\
 \hline
 $\times 2$  \\
 \hline \hline
 nn.Linear(128, 2)  \\
 \hline
 nn.Softmax \\
 \hline
\end{tabular}
\quad
\begin{tabular}{||c||} 
 \hline
 Discriminator  \\ [0.5ex] 
 \hline\hline
 nn.Linear(128, 256)  \\ 
 \hline
 nn.ReLU  \\ 
 \hline \hline 
  nn.Linear(256, 256)  \\ 
 \hline
 nn.ReLU  \\
 \hline
 $\times$5\\
 \hline\hline
 nn.Linear(256, 2)  \\
 \hline
 nn.Softmax \\
 \hline
\end{tabular}
\end{center}     



\subsubsection{Evaluation Metrics}
\mypara{Unsupervised Accuracy Estimation. } We use the absolute estimation error $|\hat{\acc}-\acc(f,U)|$ to measure the performance of the accuracy estimator $\hat{\acc}$. 

\mypara{Error Detection. } We formulate error detection as a binary classification problem (a test point detected to be mis-classified is regarded as in the positive class). Then we can use F1 score to measure the performance of the algorithms quantitatively.

\subsubsection{Baselines}
We consider the following baselines:

\mypara{Proxy Risk. } We consider proxy risk method~\citep{chuang2020estimating} as a baseline for both unsupervised accuracy estimation and error detection tasks. In proxy risk method, they train the check model and fine-tune it to maximize the disagreement using a separate target training dataset sampled from the distribution of the test dataset $U_X$. However, in practice, we might not have such a training dataset. Thus, in our implementation, we will use $U_X$ as the training dataset if a separate target training dataset is not available. Specifically, for Digits, we use a separate target training dataset, while for other dataset categories, we use $U_X$ as the target training dataset. Note that our method doesn't need a separate target training dataset. We follow their settings to set the hyperparameters. Specifically, for all experiments, we set $\lambda=50$ and $T_2=20$. In Digits and CIFAR10-C experiments, we set $\alpha=0.1$ and $\epsilon=1.1\cdot \epsilon_0$; in Office-31 and iWildCam experiments, we set $\alpha=1$ and $\epsilon=1.05\cdot \epsilon_0$, where $\epsilon_0$ is the value of divergence loss $R_S(f'g')+\alpha d(p_S^{g'}(Z), p_T^{g'}(Z))$ computed using the pre-trained check model $h'=f'g'$. The way to pre-train the check model $h'$ is the same as the DANN training method in Section~\ref{sec:model-arch-training-config}. For fine-tuning the check model to maximize the disagreement, we use SGD optimizer for CIFAR10-C and Adam optimizer for other datasets. The learning rate is set to be $10^{-3}$ in Digits, $10^{-6}$ in Office-31, $10^{-4}$ in CIFAR10-C, $10^{-5}$ in iWildCam and $10^{-3}$ in Amazon Review.      

\mypara{Avg Conf. } We consider the average confidence score as a baseline for unsupervised accuracy estimation task. We know if the model is well calibrated on $U_X$, then the average confidence of the model on $U_X$ can represent the accuracy. \cite{elsahar2019annotate} propose to use the average confidence score as a metric to measure the performance drop of a model. In our context, it is equivalent to using the average confidence as a measure for the accuracy.   

\mypara{Ens Avg Conf. } Deep ensemble~\citep{lakshminarayanan2017simple} has been shown to be an effective technique to improve the model calibration. So we consider the average confidence score of an ensemble model as a baseline for unsupervised accuracy estimation task. We use the same architecture and training procedure of the pre-trained model $f$ to train the models in the ensemble. The number of models in the ensemble is 10.   

\mypara{MSP. } We consider Maximum Softmax Probability (MSP)~\citep{hendrycks2016baseline} as a baseline for error detection. We pick the confidence threshold using a test dataset $\Dtest$ sampled from the training distribution $P_{X, Y}$ such that the fraction of data points in $\Dtest$ whose confidence scores are less than the threshold is equal to the error of the given model on $\Dtest$.  

\mypara{Trust Score. } We consider Trust Score~\citep{jiang2018trust} as a baseline for error detection task. We use the logit layer as the input to Trust Score. And we pick the threshold using a test dataset $\Dtest$ sampled from the training distribution $P_{X, Y}$ such that the fraction of data points in $\Dtest$ whose trust scores are less than the threshold is equal to the error of the given model on $\Dtest$.

\subsection{Full Plots for Unsupervised Accuracy Estimation}
\label{sec:plot-accuracy-estimation}
We plot the results for using typical DNN as the model $f$'s architecture in Figure~\ref{fig:accuracy-estimation-s-sl} and the results for using DANN-arch as the model $f$'s architecture in Figure~\ref{fig:accuracy-estimation-c-sl}. 

\begin{figure*}[t!]
    \centering
    \begin{subfigure}{\scalefactor\linewidth}
	    \centering
		\includegraphics[width=\linewidth]{figures/acc_estimation/digits_s_sl.pdf}
	\end{subfigure}
	\begin{subfigure}{\scalefactor\linewidth}
		\centering
		\includegraphics[width=\linewidth]{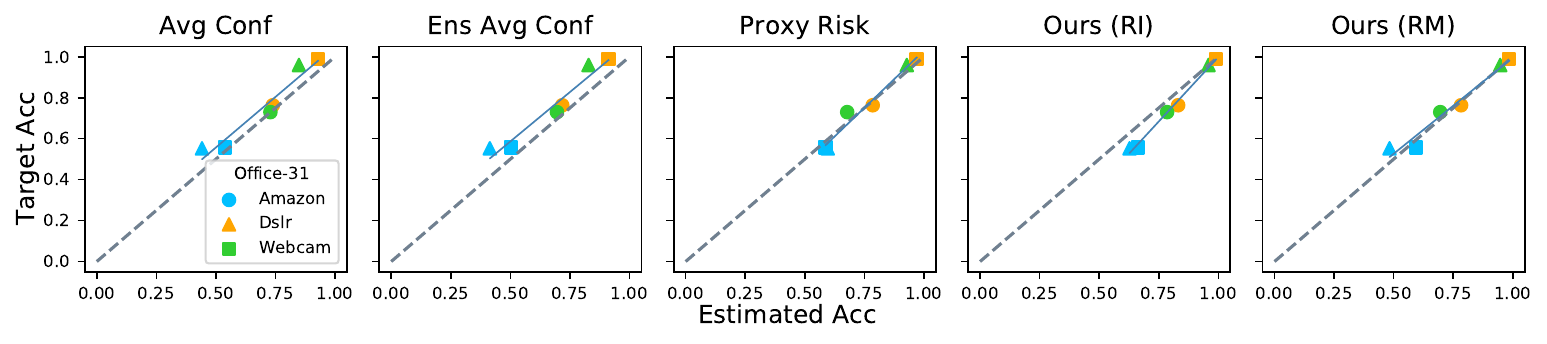}
	\end{subfigure} 
    \begin{subfigure}{\scalefactor\linewidth}
		\centering
		\includegraphics[width=\linewidth]{figures/acc_estimation/cifar_s_sl.pdf}
	\end{subfigure} 
	\begin{subfigure}{\scalefactor\linewidth}
		\centering
		\includegraphics[width=\linewidth]{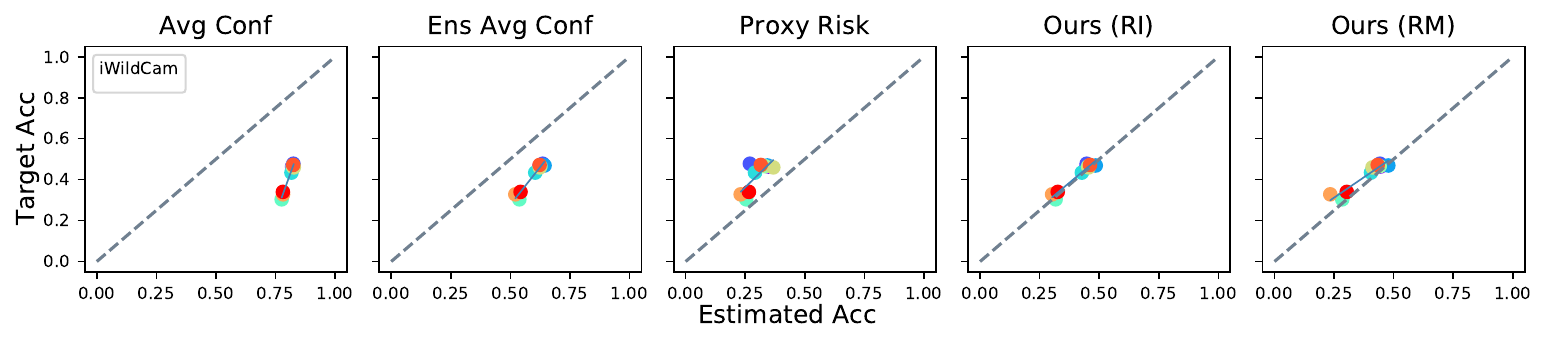}
	\end{subfigure} 
	\begin{subfigure}{\scalefactor\linewidth}
	    \centering
		\includegraphics[width=\linewidth]{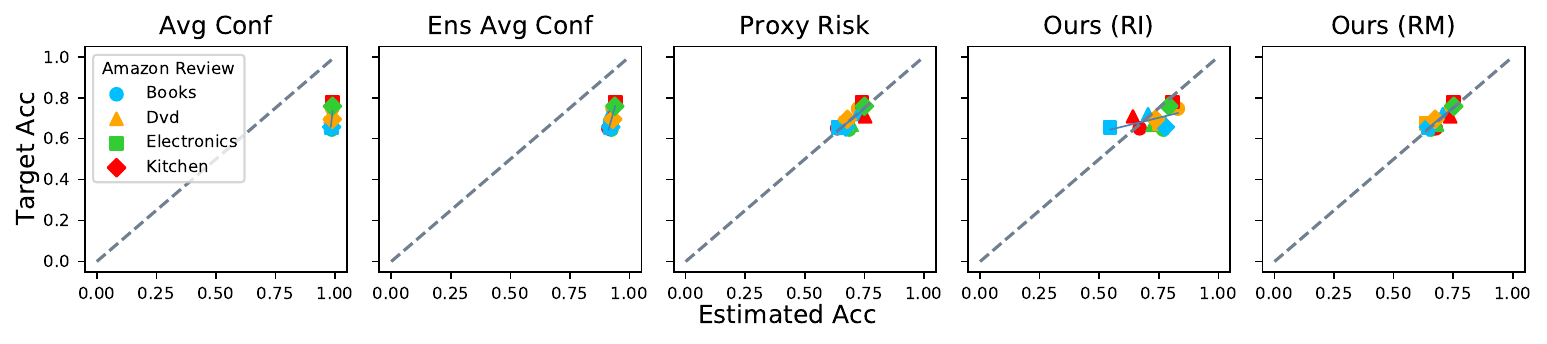}
	\end{subfigure}
	\caption{\small Accuracy estimation results for each dataset pair. We use typical DNN as the architecture for the model $f$. We use symbols to represent training datasets and colors to represent test datasets. For CIFAR10-C and iWildCam, there is only one training dataset with multiple test datasets. The dashed line represents perfect prediction (target accuracy = estimated accuracy). Points beneath (above) the dashed line indicate overestimation (underestimation). The solid lines are regression lines of the results. 
	}
	\label{fig:accuracy-estimation-s-sl}
	\vspace{-0.3cm}
\end{figure*}

\begin{figure*}[t!]
    \centering
    \begin{subfigure}{\scalefactor\linewidth}
	    \centering
		\includegraphics[width=\linewidth]{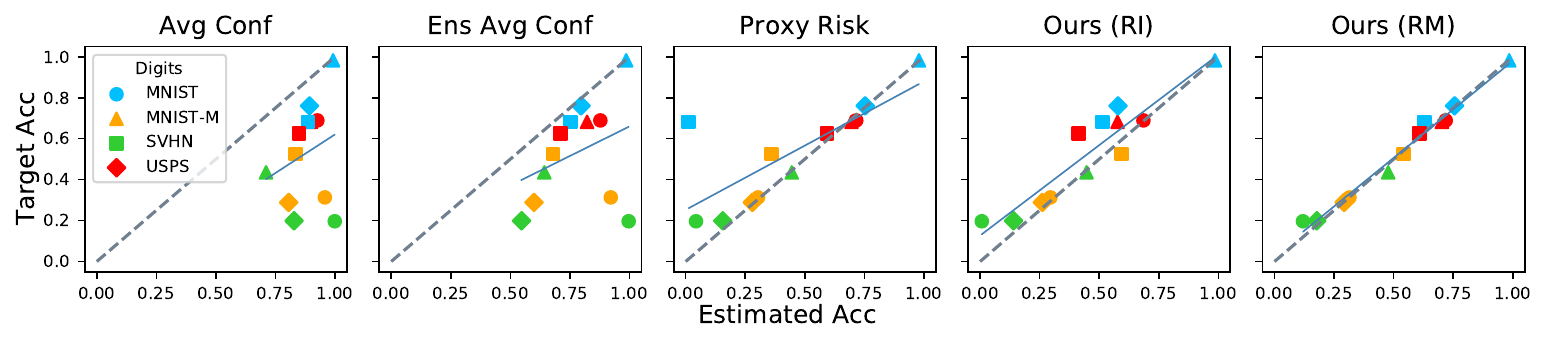}
	\end{subfigure}
	\begin{subfigure}{\scalefactor\linewidth}
		\centering
		\includegraphics[width=\linewidth]{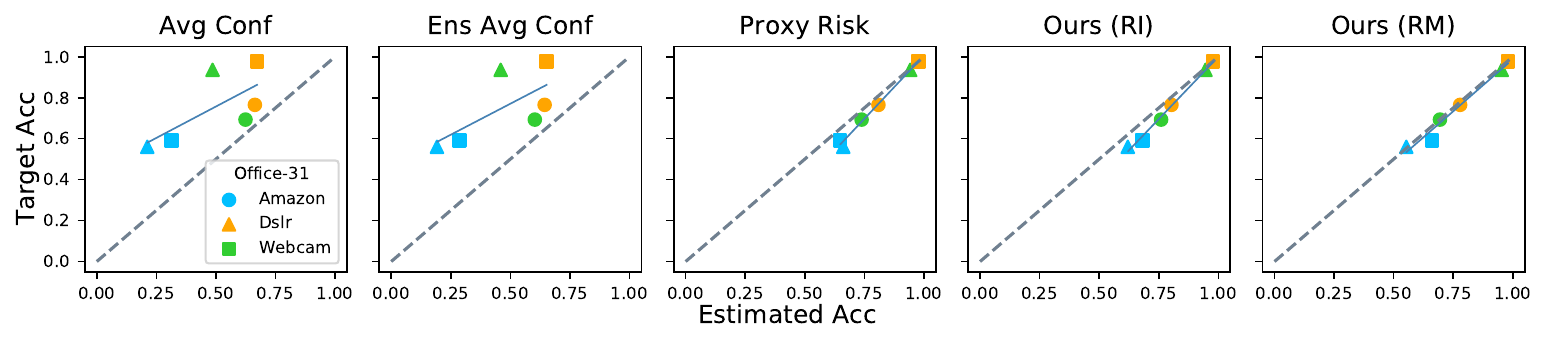}
	\end{subfigure} 
    \begin{subfigure}{\scalefactor\linewidth}
		\centering
		\includegraphics[width=\linewidth]{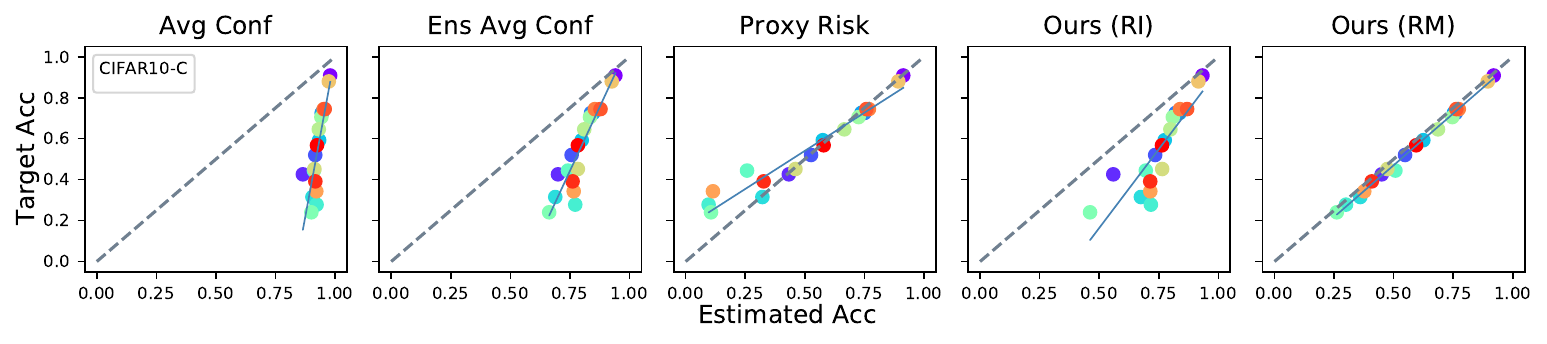}
	\end{subfigure} 
	\begin{subfigure}{\scalefactor\linewidth}
		\centering
		\includegraphics[width=\linewidth]{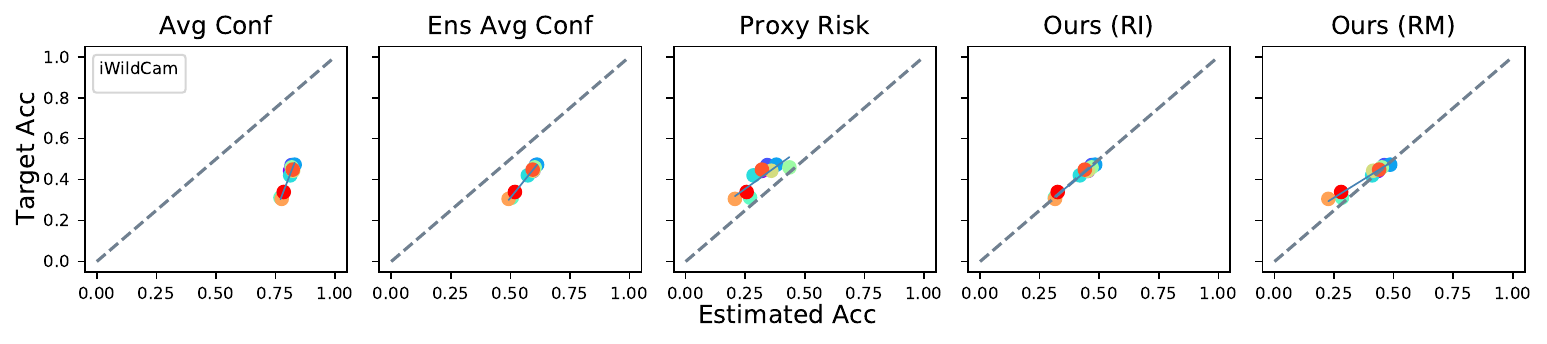}
	\end{subfigure} 
	\begin{subfigure}{\scalefactor\linewidth}
		\centering
		\includegraphics[width=\linewidth]{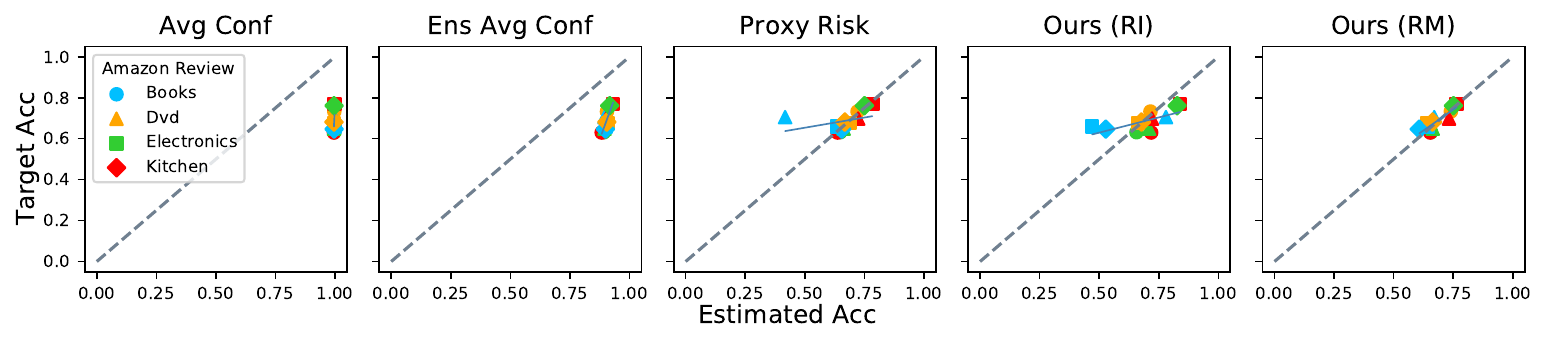}
	\end{subfigure} 
	\caption{\small Accuracy estimation results for each dataset pair. We use DANN-arch as the architecture for the model $f$. We use symbols to represent training datasets and colors to represent test datasets. For CIFAR10-C and iWildCam, there is only one training dataset with multiple test datasets. The dashed line represents perfect prediction (target accuracy = estimated accuracy). Points beneath (above) the dashed line indicate overestimation (underestimation). The solid lines are regression lines of the results.  
	}
	\label{fig:accuracy-estimation-c-sl}
	\vspace{-0.3cm}
\end{figure*}

\subsection{Multiple Runs of Experiments}
\label{sec:multiple-runs}

Since proxy risk method and our method require training and fine-tuning the models, there might be some variance in the results. Thus, we repeat each experiment in Table~\ref{tab:main-results} for proxy risk method and our method five times, and show the box plots of the results measured by average absolute estimation error and average F1 score (See Figure~\ref{fig:results-error-bar}). The results show that our method is consistently better than proxy risk method and the variance of the results of our method is generally smaller than that of proxy risk method. 

\begin{figure*}[t!]
    \centering
    \begin{subfigure}{\linewidth}
	    \centering
		\includegraphics[width=\linewidth]{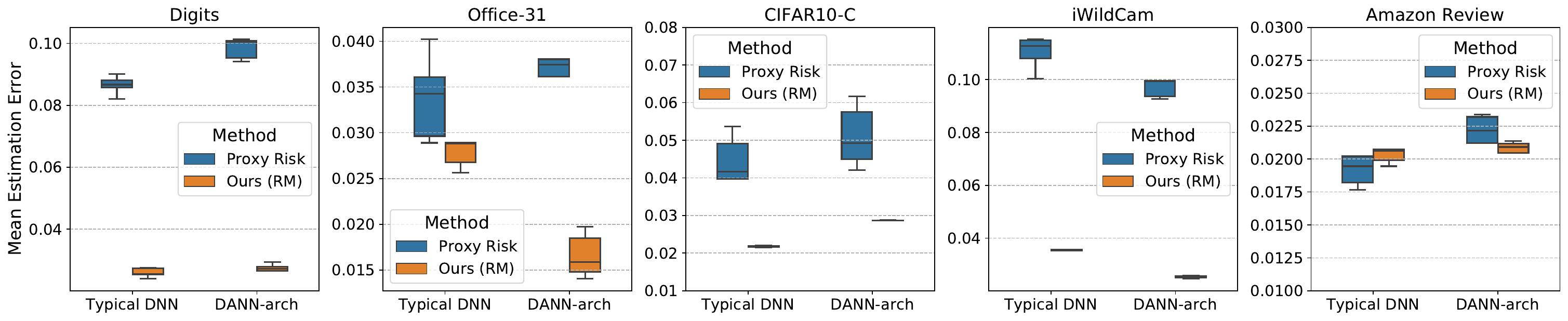}
		\caption{Unsupervised Accuracy Estimation}
	\end{subfigure}
	\begin{subfigure}{\linewidth}
		\centering
		\includegraphics[width=\linewidth]{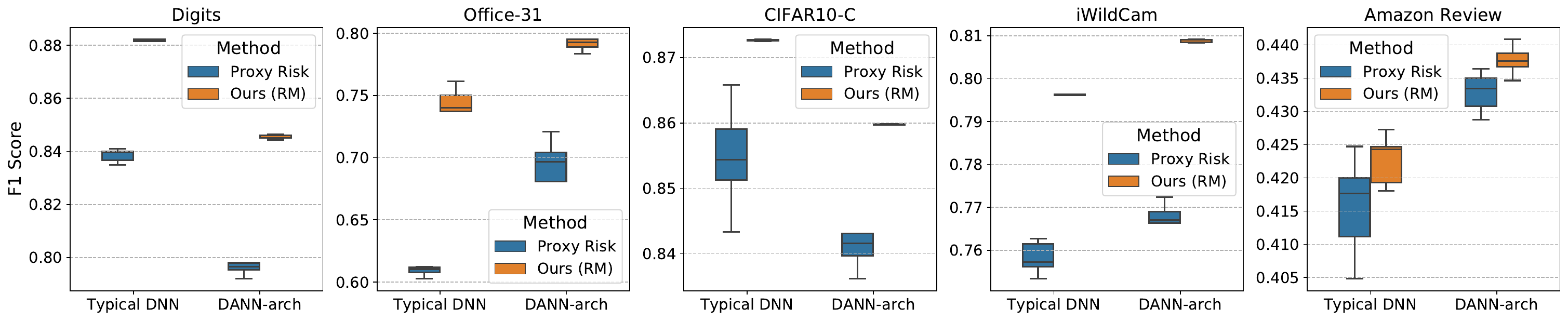}
		\caption{Error Detection} 
	\end{subfigure} 
	\vspace{-0.2cm}
	\caption{\small Results for multiple runs of experiments. 
	}
	\label{fig:results-error-bar}
\end{figure*}

\subsection{Validating the Theoretical Analysis} \label{app:validate-theory}

Our analysis relies on three conditions: (A) The ensemble models make small error on the test inputs correctly classified by $f$ (and thus tend to over-estimate the accuracy); (B) The ensemble models mostly disagree with $f$ on $R_X$; (C) The ensemble models have large diversity on $B_X$. Since we don't use threshold in our implementation, $B_X$ here is the data points in $W_X \setminus R_X$ where the prediction of the ensemble (via majority vote) is wrong. Our experiments on six dataset pairs over two dataset categories (MNIST$\rightarrow$MNIST-M, MNIST$\rightarrow$USPS and MNIST-M$\rightarrow$USPS on the Digits dataset category; Amazon$\rightarrow$Dslr, Dslr$\rightarrow$Webcam, and Amazon$\rightarrow$Webcam on the Office-31 dataset category) show that empirically they are roughly satisfied.  

Table~\ref{tab:validate-theory} shows the actual accuracy, the estimated accuracy obtained via the initial ensemble (Estimated Acc w/o self-training), and the estimated accuracy after applying the self-training (Estimated Acc w/ self-training). The comparison confirms the over-estimation and shows the self-training can rectify that. The table also shows the upper bound $\tilde{\nu}$ on average error on correct points $\nu$ over all iterations is small. The upper bound $\tilde{\gamma}$ on the average probability of agreement $\gamma$ between the ensemble models and $f$ on $R_X$ over all iterations is close to 0, so the ensemble models mostly disagree with $f$ on $R_X$. Finally, the lower bound $\sigma_L^2$ on the diversity of the ensemble $\sigma^2$ over all iterations is relatively large. 

Besides, to support our claims regarding pseudo-labels, we perform experiments (using the same settings as the experiments in Table~\ref{tab:validate-theory}) on MNIST$\rightarrow$MNIST-M, our observations are: (1) the pseudo-labels are not all correct. The accuracies of the pseudo-labels over three iterations are 93.89\%, 94.02\% and 93.82\%. (2) the new ensemble will become more accurate with the help of correct pseudo-labels. The accuracies of the ensembles on $U$ over three iterations are 89.56\%, 93.51\% and 94.53\%. (3) the new ensemble will be less diverse on the pseudo-labeled data. The upper bound of average $\sigma_x^2$ on the pseudo-labeled data over all iterations is 4.78\%, which is relatively small compared to $\sigma_L^2=26.54\%$. 

Furthermore, we perform an experiment (using the same settings as the experiments in Table~\ref{tab:validate-theory}) on MNIST$\rightarrow$MNIST-M to verify if our conditions still hold when the pre-trained classifier is a majority vote over an ensemble. The ensemble contains 10 models trained on $\Dtr$ with different random initializations. The results are similar to those for a single model. Our success conditions are still roughly satisfied and the accuracy estimation is great: $\tilde{\nu}=2.61\%$, $\tilde{\gamma}=0.59\%$,  $\sigma_L^2=28.94\%$ and the estimation error is 0.0009. 

\begin{table*}[t!]
    \centering
		\begin{tabular}{c|ccc|ccc}
			\toprule
		    Dataset Category &  \multicolumn{3}{c|}{Digits}  & \multicolumn{3}{c}{Office-31}  \\ \hline
	        Dataset Pair & M$\rightarrow$MM & M$\rightarrow$U & MM$\rightarrow$U & A$\rightarrow$D & D$\rightarrow$W & A$\rightarrow$W  \\ \hline 
	       Actual Acc & 27.19 & 67.56 & 60.04 & 76.31 & 95.97 & 72.96 \\
	       Estimated Acc w/o self-training  & 33.05 & 70.30 & 69.76 & 80.52 & 95.85 & 74.34 \\
	       Estimated Acc w/ self-training  & 27.50 & 68.21 & 64.28 & 78.11 & 95.09 & 70.57 \\
	       \hline
	       $\tilde{\nu}$ & 3.15 & 2.18 & 4.22 & 6.16 & 2.17 & 11.11 \\
	       $\tilde{\gamma}$ & 0.57 & 0.90 & 3.82 & 1.92 & 0.20 & 0.29 \\
	       $\sigma_L^2$ & 26.54 & 12.89 & 24.57 & 15.61 & 8.80 & 14.67 \\
			\bottomrule
		\end{tabular}
	\caption[]{\small Empirical results to support the theoretical analysis. We use typical DNN as the architecture for the model $f$. M is MNIST, MM is MNIST-M, U is USPS, A is Amazon, D is Dslr and W is Webcam. The ensemble training algorithm we use is $\ens_\textrm{RM}$. On Digits, we use $N=10$ and $T=3$ while on Office-31, we use $N=15$ and $T=2$. All values are percentages.}
	\label{tab:validate-theory}
\end{table*}

\subsection{Evaluation on Pre-trained Models with Various Architectures}
\label{app:evaluate-various-arch}
We evaluate our method on pre-trained models with different deep learning model architectures on Digits. The architectures considered are Convolutional Neural Network (CNN)~\cite{lecun1998gradient}, Convolutional Neural Network with Batch Normalization (CNN-BN), ResNet18, ResNet34~\cite{he2016deep}, DenseNet40 and DenseNet100~\cite{huang2017densely}. The results in Table~\ref{tab:eval-various-model-arch} demonstrate that our method consistently outperforms other methods on pre-trained models with various architectures. 

\begin{table}[t]
    \begin{adjustbox}{width=0.9\columnwidth,center}
		\begin{tabular}{l|l|c|l|c}
			\toprule
		   Task & \multicolumn{2}{|c}{Accuracy Estimation} & \multicolumn{2}{|c}{Error Detection} \\ \hline
		   Architecture of $f$ & Method & Estimation Error $\downarrow$ & Method & F1 score $\uparrow$  \\ \hline \hline 
		    \multirow{5}{0.13\linewidth}{CNN} & Avg Conf &  0.252$\pm$0.159 & MSP  & 0.595$\pm$0.228 \\ 
	       & Ens Avg Conf  &  0.141$\pm$0.111 & Trust Score &  0.567$\pm$0.125 \\ 
	       & Proxy Risk &  0.080$\pm$0.167 & Proxy Risk & 0.817$\pm$0.129 \\ 
	       & Ours (RI) &  0.120$\pm$0.125 & Ours (RI) & 0.735$\pm$0.170 \\
	       & Ours (RM) &  {\bf 0.020}$\pm$0.022 & Ours (RM) & {\bf 0.865}$\pm$0.075 \\ \hline
	       \multirow{5}{0.13\linewidth}{CNN-BN} & Avg Conf & 0.404$\pm$0.180 & MSP  &  0.467$\pm$0.195  \\ 
	       & Ens Avg Conf  &  0.337$\pm$0.229 & Trust Score &  0.496$\pm$0.195 \\ 
	       & Proxy Risk & 0.085$\pm$0.142 & Proxy Risk & 0.844$\pm$0.118  \\ 
	       & Ours (RI) &  0.164$\pm$0.218 & Ours (RI) & 0.698$\pm$0.235 \\
	       & Ours (RM) & {\bf 0.023}$\pm$0.020 & Ours (RM) & {\bf 0.881}$\pm$0.084 \\ \hline
	       \multirow{5}{0.13\linewidth}{ResNet18} & Avg Conf &  0.434$\pm$0.330 & MSP  &  0.365$\pm$0.193 \\ 
	       & Ens Avg Conf  &  0.245$\pm$0.220 & Trust Score &  0.408$\pm$0.176  \\ 
	       & Proxy Risk & 0.111$\pm$0.165 & Proxy Risk & 0.787$\pm$0.178 \\ 
	       & Ours (RI) &  0.173$\pm$0.154 & Ours (RI) & 0.684$\pm$0.246  \\
	       & Ours (RM) & {\bf 0.039}$\pm$0.043 & Ours (RM) & {\bf 0.834}$\pm$0.128 \\ \hline
	       \multirow{5}{0.13\linewidth}{ResNet34} & Avg Conf &  0.399$\pm$0.261  & MSP  &  0.484$\pm$0.114 \\ 
	       & Ens Avg Conf  &  0.288$\pm$0.242 & Trust Score &  0.551$\pm$0.153  \\ 
	       & Proxy Risk &   0.115$\pm$0.196 & Proxy Risk & 0.795$\pm$0.179 \\ 
	       & Ours (RI) &  0.181$\pm$0.170 & Ours (RI) & 0.670$\pm$0.168  \\
	       & Ours (RM) & {\bf 0.042}$\pm$0.043 & Ours (RM) & {\bf 0.834}$\pm$0.129  \\ \hline
	       \multirow{5}{0.13\linewidth}{DenseNet40} & Avg Conf &  0.389$\pm$0.309 & MSP  &  0.470$\pm$0.097 \\ 
	       & Ens Avg Conf  &  0.199$\pm$0.205 & Trust Score &  0.560$\pm$0.064 \\ 
	       & Proxy Risk &  0.115$\pm$0.168 & Proxy Risk & 0.794$\pm$0.185 \\ 
	       & Ours (RI) &  0.170$\pm$0.154 & Ours (RI) &  0.697$\pm$0.239 \\
	       & Ours (RM) &  {\bf 0.041}$\pm$0.046 & Ours (RM) & {\bf 0.833}$\pm$0.144 \\ \hline
	       \multirow{5}{0.13\linewidth}{DenseNet100} & Avg Conf &  0.388$\pm$0.350 & MSP  &  0.430$\pm$0.229 \\ 
	       & Ens Avg Conf  &  0.248$\pm$0.254 & Trust Score & 0.533$\pm$0.155 \\ 
	       & Proxy Risk &  0.127$\pm$0.176 & Proxy Risk & 0.759$\pm$0.200 \\ 
	       & Ours (RI) & 0.186$\pm$0.182  & Ours (RI) & 0.677$\pm$0.256  \\
	       & Ours (RM) & {\bf 0.047}$\pm$0.052 & Ours (RM) & {\bf 0.803}$\pm$0.160 \\ 
			\bottomrule
		\end{tabular}
	\end{adjustbox}
	\caption[]{\small Evaluate the methods for unsupervised accuracy estimation and error detection on pre-trained models with various model architectures. The models are trained using supervised learning method. We show the mean and standard deviation of absolute estimation error and F1 score (mean$\pm$std). The numbers are calculated over the training-test dataset pairs constructed in the Digits dataset category. {\bf Bold} numbers are the superior results. } 
	\label{tab:eval-various-model-arch}
\end{table}

\subsection{Implementation of the Framework with Explicit Thresholding} 
\label{sec:thresholding-implementation-results}

We also implement the Framework~\ref{alg:main-framework} with explicit thresholding. We use $\ens_\textrm{RM}$ to construct the ensemble $\{h_i\}_{i=1}^N$ and use the following pseudo-labeling strategy: for each $x\in R_X$, if the majority vote of the ensemble on $x$ is different from $f(x)$, then we use the majority vote as the pseudo-label; otherwise, we use a random label that is different from $f(x)$ as the pseudo-label. We perform experiments on Digits dataset category to compare the method using thresholding with Algorithm~\ref{alg:main} where we don't use thresholding. The results in Table~\ref{tab:thresholding-comparison-results} show that using explicit thresholding leads to similar performance as Algorithm~\ref{alg:main}.

\begin{table}
\centering
		\begin{tabular}{c|c|c|c}
			\toprule
			\multicolumn{2}{c|}{Method} & Accuracy Estimation & Error Detection \\ \hline
			Thresholding & Threshold $\tau$ & Estimation Error $\downarrow$ & F1 score $\uparrow$ \\ \hline \hline
			 Yes & 0.5 & 0.026$\pm$0.023 & 0.881$\pm$0.086 \\ \hline 
	         Yes & 0.4 & 0.033$\pm$0.030 & 0.877$\pm$0.086 \\ \hline 
	         No & - &  0.023$\pm$0.020 & 0.881$\pm$0.084 \\
		   \bottomrule
		\end{tabular}
 	\captionof{table}{\small Results for comparing our methods with and without thresholding. We use typical DNN as the architecture for the model $f$. We show the mean and standard deviation of absolute estimation error and F1 score (mean$\pm$std). The numbers are calculated over the training-test dataset pairs in Digits dataset category. } 
 	\label{tab:thresholding-comparison-results}
\end{table}

\subsection{Ablation Study on CIFAR10-C}
\label{sec:extra-ablation-studies}

\begin{figure}[t!]
    \centering
    \begin{subfigure}{0.3\linewidth}
	    \centering
		\includegraphics[width=\linewidth]{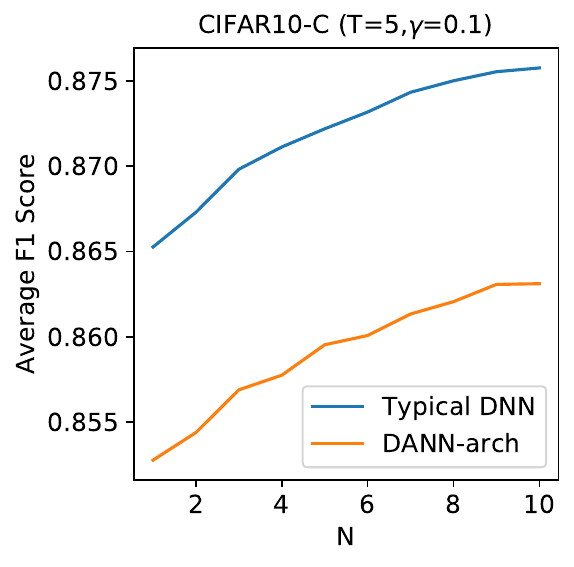}
	\end{subfigure}
	\begin{subfigure}{0.3\linewidth}
		\centering
		\includegraphics[width=\linewidth]{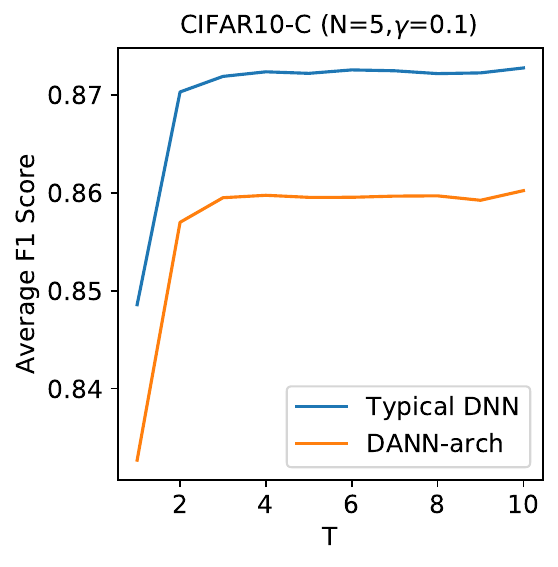}
	\end{subfigure} 
    \begin{subfigure}{0.3\linewidth}
		\centering
		\includegraphics[width=\linewidth]{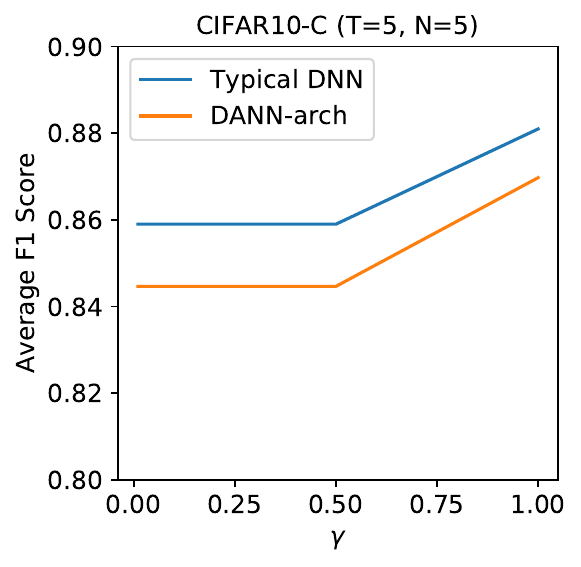}
	\end{subfigure} 
	\caption{\small Ablation study for the effect of ensemble and self-training techniques on CIFAR10-C. $N$ is the number of models in the ensemble, $T$ is the number of self-training iterations, and $\gamma$ is the weighting parameter for the loss term on the pseudo-labeled data. The ensemble training algorithm we use is $\ens_\textrm{RM}$. 
	}
	\label{fig:ablation-study-cifar10-c}
\end{figure}

We perform an additional ablation study on CIFAR10-C and the results are shown in Figure~\ref{fig:ablation-study-cifar10-c}. We observe a similar trend as that on Digits. 

\subsection{Analysis on Target Accuracy}
\label{sec:analysis-target-acc}
\begin{table}
\centering
\begin{adjustbox}{width=\columnwidth,center}
		\begin{tabular}{l|l|c|c|c|c}
			\toprule
			\multirow{2}{0.1\linewidth}{Dataset Category} & \multirow{2}{0.1\linewidth}{Dataset Pair} & \multirow{2}{0.12\linewidth}{Target Acc. of $f$} &  \multicolumn{3}{|c}{Ours (RM)} \\ \cline{4-6}
			 &  &  & Target Acc. of $\{h_i\}_{i=1}^N$ & Estimation Error & F1 score  \\ \hline \hline
			\multirow{3}{0.13\linewidth}{Digits} & M$\rightarrow$MM & 27.19\% & 93.99\% & 0.0013 & 0.9901 \\
			& M$\rightarrow$U & 67.56\% & 92.53\% & 0.0065 & 0.9418 \\
			& S$\rightarrow$U & 54.46\% & 67.46\% & 0.0010 & 0.7941 \\ \hline
			\multirow{3}{0.13\linewidth}{iWildCam} & 0$\rightarrow$1 & 46.40\% & 39.66\% &  0.0202 & 0.7872 \\
			& 0$\rightarrow$2 & 46.90\% & 42.73\% & 0.0335 & 0.7996 \\
			& 0$\rightarrow$9 & 34.01\% & 24.50\% & 0.0376 & 0.7871 \\
		   \bottomrule
		\end{tabular}
\end{adjustbox}
 	\captionof{table}{\small Results of comparing the target accuracy of the pre-trained model $f$ and the ensemble check models $\{h_i\}_{i=1}^N$. We use typical DNN as the architecture for the model $f$. The prediction of the ensemble $\{h_i\}_{i=1}^N$ is produced via majority vote. M is MNIST, MM is MNIST-M, U is USPS and S is SVHN. } 
 	\label{tab:comparing-acc-of-f-h}
\end{table}

The target accuracy of our ensemble check models can be higher or lower than that of the pre-trained model $f$ depending on the datasets and in both cases, our method can achieve good performance. In Table~\ref{tab:comparing-acc-of-f-h}, we compare the target accuracy of the pre-trained model $f$ and the ensemble check models $\{h_i\}_{i=1}^N$ generated by our method with $\ens_\textrm{RM}$ on some dataset pairs in Digits and iWildCam. The results show that on Digits where the domain adaptation method can work well, the target accuracy of $\{h_i\}_{i=1}^N$ built by $\ens_\textrm{RM}$ is usually higher than that of the pre-trained model $f$. In such cases, the high accuracy of $\{h_i\}_{i=1}^N$ plays an important role in the good performance of our method. However, on iWildCam where the domain adaptation method fails, the target accuracy of $\{h_i\}_{i=1}^N$ is typically not higher than that of $f$. In such cases, our method can still achieve good performance due to the diversity of the ensemble. 


\subsection{Analysis on Proxy Risk}
\label{sec:analysis-proxy-risk}
Proxy Risk has two stages: first train a check model using domain-invariant representations (DIR) and then fine-tune it to maximize the disagreement between the pre-trained model $f$ and the check model $h$ on the target data while maintaining small DIR loss. To study the effect of the disagreement maximization in Proxy Risk, we perform an experiment on Digits for the typical DNN $f$ by removing the disagreement maximization component from Proxy Risk. The results show that without disagreement maximization, the mean F1 score of Proxy Risk on Digits will decrease from 0.844 to 0.812. 

Besides, to see whether combining Proxy Risk with ensemble could lead to better results than our method, we perform an experiment on Digits for a variant of Proxy Risk: the check model is an ensemble of models $h'_1,\dots,h'_t$ via majority vote and each model $h'_j$ in the ensemble is fine-tuned from the pre-trained check model $h'$ using Proxy Risk's maximizing disagreement training objective with different randomness. We set $t=5$ and use typical DNN as the model architecture of $f$. The experimental result on Digits is: the mean absolute estimation error is 0.0814 and the mean F1 score is 0.8515. The performance of this variant of Proxy Risk is worse than that of our method with $\ens_\textrm{RM}$ (the result for our method is: mean estimation error is 0.0230 and mean F1 score is 0.8810).

\end{document}